\documentclass[lettersize,journal]{IEEEtran}
\usepackage{amsmath,amsfonts}
\usepackage{algorithmic}
\usepackage{algorithm}
\usepackage{array}
\usepackage[caption=false,font=normalsize,labelfont=sf,textfont=sf]{subfig}
\usepackage{textcomp}
\usepackage{stfloats}
\usepackage{url}
\usepackage{verbatim}
\usepackage{graphicx}
\usepackage{cite}
\usepackage{xspace}
\usepackage{color}
\usepackage{booktabs}
\usepackage{tcolorbox}
\usepackage{arydshln}

\usepackage{enumitem}
\usepackage{multirow}
\usepackage{amsmath,amsfonts,bm,amsthm}
\usepackage{amssymb}

\newtheorem{proposition}{Proposition}
\newtheorem{remark}{Remark}
\newtheorem{theorem}{Theorem}

\usepackage{xcolor}

\usepackage{amsmath,amsfonts,bm,amsthm}
\renewcommand\vec[1]{\ensuremath\boldsymbol{#1}}

\newcommand{\mX}{\bm{X}}

\newcommand{\vx}{\bm{x}}

\newcommand{\mbrp}[1]{\mathbb{R}_{+}^{#1}}
\newcommand{\mbr}[1]{\mathbb{R}^{#1}}

\newcommand{\tAnb}{\mathcal{A}}

\newcommand{\idx}[1]{\mathcal{I}_{#1}}

\newcommand{\vu}{\bm{u}}

\newcommand{\vpsi}{\boldsymbol{\psi}}

\newcommand{\mPsi}{\vec{\Psi}}

\DeclareMathOperator*{\argmin}{arg\,min}

\DeclareMathOperator*{\softming}{SoftMin_{\gamma}}

\DeclareMathOperator*{\softminsel}{SoftMinSel_{\gamma}}

\newcommand{\mPi}{{\boldsymbol\Pi}}

\newcommand{\vw}{\boldsymbol{w}}

\def\eg{\emph{e.g.}}

\newcommand{\cov}{\boldsymbol{\Sigma}}

\newcommand{\mPhi}{\boldsymbol{\Phi}}

\newcommand{\mM}{\boldsymbol{M}}

\newcommand{\mD}{\boldsymbol{D}}

\newcommand{\vmu}{\boldsymbol{\mu}}

\newcommand{\stkout}[1]{{\ifmmode\text{\sout{\ensuremath{#1}}}\else\sout{#1}\fi}}

\newcommand{\commentt}[1]{}

\makeatletter
\DeclareRobustCommand\onedot{\futurelet\@let@token\@onedot}
\def\@onedot{\ifx\@let@token.\else.\null\fi\xspace}

\def\eg{\emph{e.g}\onedot} 
\def\ie{\emph{i.e}\onedot} 
 
 \def\vs{\emph{vs}\onedot}
\def\wrt{w.r.t\onedot} 

\makeatother

\begin{document}

\title{Uncertainty-DTW for Sequences and Visual Tokens %Uncertainty-Aware Alignment: A Probabilistic Framework for Sequences and Visual Tokens
}

\author{Lei Wang, Syuan-Hao Li, Yongsheng Gao, Piotr Koniusz
\thanks{
This work is supported in part by the Australian Research Council (ARC) under Industrial Transformation Research Hub Grant IH180100002 (Corresponding authors: Yongsheng Gao and Piotr Koniusz).

Lei Wang, Syuan-Hao Li, and Yongsheng Gao are with the School of Engineering and Built Environment, Electrical and Electronic Engineering, Griffith University (email: l.wang4@griffith.edu.au, syuan-hao.li@griffithuni.edu.au, yongsheng.gao@griffith.edu.au).

% Syuan-Hao Li is with the School of Engineering and Built Environment, Department of Electrical and Electronic Engineering, Griffith University (email: syuan-hao.li@griffithuni.edu.au).

Piotr Koniusz is with School of Computer Science and Engineering, University of New South Wales (email: piotr.koniusz@unsw.edu.au).

% Yongsheng Gao is with the School of Engineering and Built Environment, Electrical and Electronic Engineering, Griffith University (email: yongsheng.gao@griffith.edu.au).

% Corresponding author: Yongsheng Gao.
}}

% The paper headers
% \markboth{Journal of \LaTeX\ Class Files,~Vol.~14, No.~8, August~2021}%
% {Shell \MakeLowercase{\textit{et al.}}: A Sample Article Using IEEEtran.cls for IEEE Journals}

% \IEEEpubid{0000--0000/00\$00.00~\copyright~2021 IEEE}
% Remember, if you use this you must call \IEEEpubidadjcol in the second
% column for its text to clear the IEEEpubid mark.

\maketitle

\begin{abstract}
Aligning structured data is a fundamental problem in computer vision and machine learning, underlying tasks such as time series analysis, human action recognition, and visual representation learning. Existing alignment methods, including Dynamic Time Warping (DTW) and its differentiable variants, rely on deterministic similarity measures and are therefore sensitive to heterogeneous and noisy features.
In this work, we introduce uncertainty-aware alignment, a probabilistic framework that models pairwise correspondences with heteroscedastic uncertainty and performs structured matching along alignment paths. Our formulation,  uncertainty-DTW (uDTW), assigns each correspondence a Normal distribution and parametrizes each alignment path by a Maximum Likelihood Estimate objective consisting of (i) a precision-weighted matching term that suppresses unreliable features, and (ii) a log-variance regularization that prevents degenerate solutions. This yields a probabilistic alignment mechanism that is robust to noise and interpretable, as uncertainty directly reflects the reliability of matches.
We further generalize this framework from temporal sequences to tokenized visual representations, enabling structured matching over sets of visual tokens. The learned uncertainty can be interpreted as a reverse-attention: semantically relevant regions exhibit low uncertainty and dominate the alignment, while ambiguous/noisy regions have high uncertainty. This provides a connection between alignment, attention, and uncertainty modeling.
We evaluate the proposed framework across diverse domains, including time series forecasting, Fr\'echet mean estimation, few-shot action recognition, and few-shot image classification on generic, fine-grained, and ultra-fine-grained benchmarks. % using vision transformer features. 
The results demonstrate consistent improvements over state-of-the-art methods and show that learned uncertainty correlates with semantic importance. These findings establish uncertainty-aware alignment as a general, robust, and interpretable framework for learning from structured data.
\end{abstract}


\section{Introduction}
\label{sec:intro}

\begin{figure}[tbp]
\vspace{-0.2cm}
\centering
\subfloat[Matched pair]{
\includegraphics[trim=0 0.5cm 0 0,clip=true,
width=0.498\linewidth]{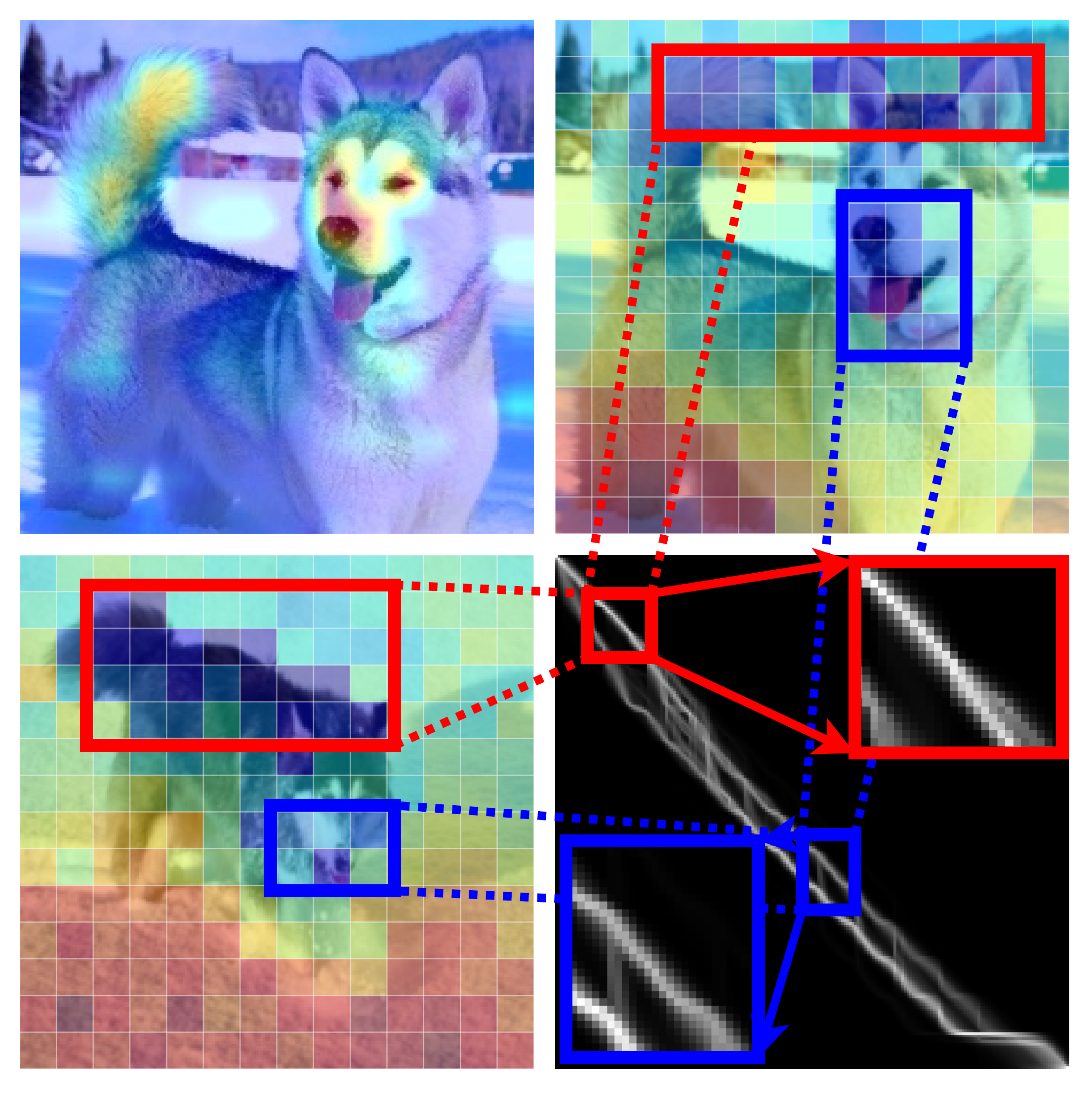}}
% \hfill
\subfloat[Unmatched pair]{%
\raisebox{0.55mm}{%
\includegraphics[trim=0 0.5cm 0 0,clip=true,
width=0.498\linewidth]{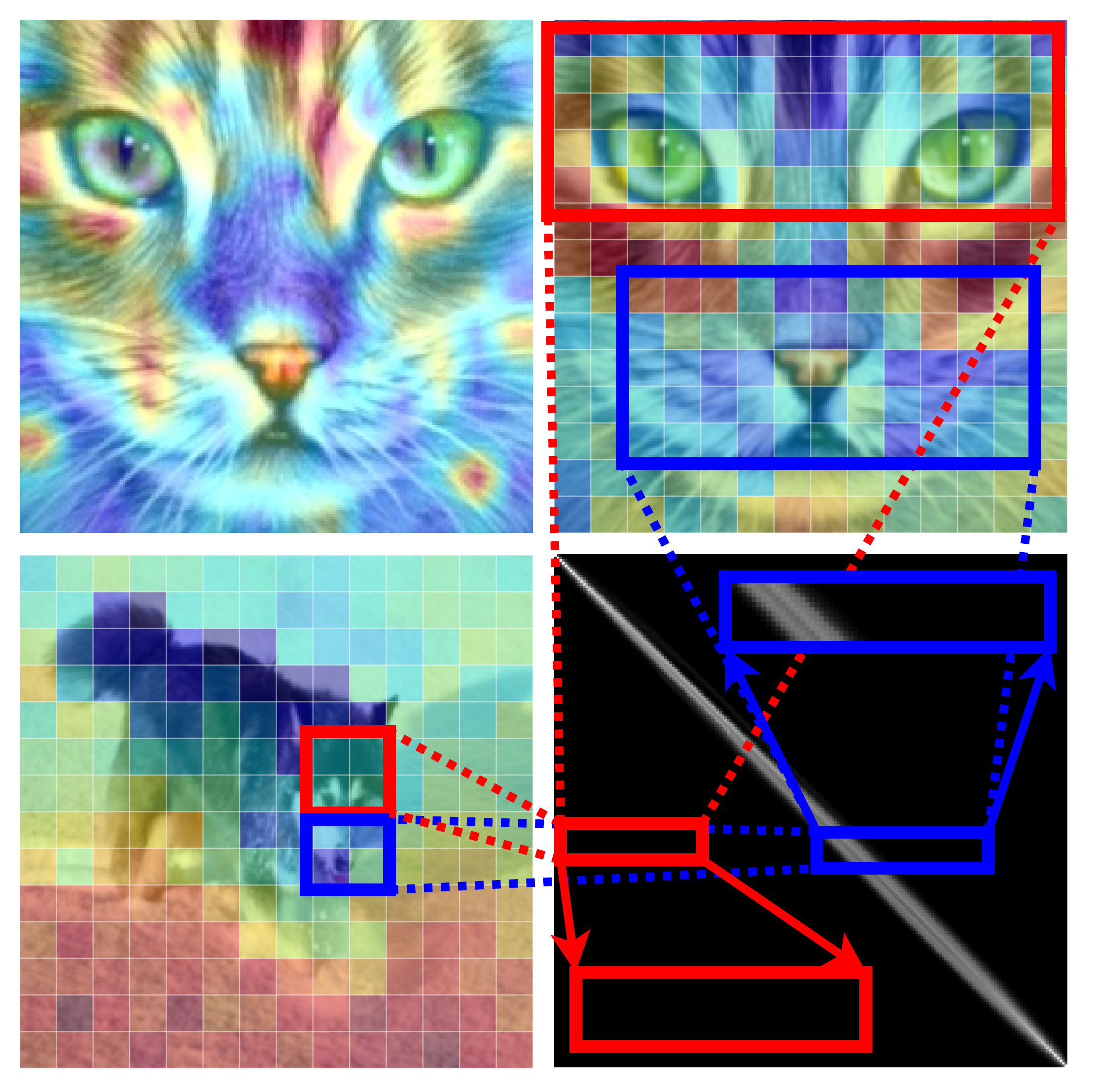}}}
\vspace{-0.1cm}
\caption{
A  connection between alignment, attention, and uncertainty: 
({\em a}) matched pair; ({\em b}) unmatched pair. In each group, the query (top right) and support (bottom left) images are overlaid with learned uncertainty (jet colormap: blue - low uncertainty, red - high uncertainty). The top left shows the DAAM attention map~\cite{daam_pr} (brighter = higher importance), and the bottom right shows uDTW alignment paths. 
Uncertainty exhibits a reverse-attention effect: low uncertainty regions align with high-attention areas and dominate the alignment, while high-uncertainty regions are suppressed. Selected patches (red/blue boxes) with zoomed-in alignment spaces further illustrate this behavior: matched pairs show actived alignment paths, especially in low uncertainty regions, whereas unmatched pairs exhibit weaker and less activated paths, indicating the absence of meaningful correspondence.
}
\label{fig:unc-attn-align}
% \vspace{-0.5cm}
\end{figure}

\IEEEPARstart{A}{ligning} structured data is a fundamental problem in computer vision and machine learning, underpinning a wide range of tasks such as time series forecasting, action recognition, and visual representation learning \cite{marco2011icml, marco2017icml, wang2022temporal, wang2022uncertainty, wang2024meet}. A central challenge in these problems is to establish reliable correspondences between elements of structured inputs, whether they are temporal samples, human body joints, or visual tokens, despite variability in dynamics, appearance, and noise. 
Among existing approaches, Dynamic Time Warping (DTW) \cite{marco2011icml} and its differentiable variants \cite{marco2017icml, pmlr-v130-blondel21a} have been widely adopted due to their ability to compute optimal alignments under flexible structural constraints. However, these methods rely on deterministic similarity measures and select correspondences based solely on minimum distance, making them sensitive to heterogeneous feature quality and prone to being dominated by noisy or unreliable observations.

% To address this limitation, o
Our prior work \cite{wang2022uncertainty} introduced uncertainty-DTW (uDTW), a probabilistic formulation of differentiable DTW that incorporates heteroscedastic aleatoric uncertainty into sequence alignment. By modeling each pairwise correspondence with a Gaussian distribution, uDTW reweights alignment costs according to feature reliability while regularizing uncertainty through a Maximum Likelihood Estimation (MLE) formulation over alignment paths. This leads to more robust alignments and improved performance on time series and skeleton-based sequence tasks. Nevertheless, this formulation was originally developed in the context of temporal sequences, and its broader implications for visual data alignment and representation learning remain underexplored.

In this work, we revisit uDTW from a broader perspective and generalize it into an uncertainty-aware alignment framework. We show that uDTW can be interpreted as a form of probabilistic structured matching, where alignment is governed by a log MLE function that jointly accounts for similarity and uncertainty. This perspective shows that uncertainty plays a central role in modulating correspondences: unreliable features are downweighted under high variance, while overconfident matches are penalized by the variance regularization term. As a result, the framework % naturally 
achieves robustness to heteroscedastic noise and provides interpretable measures of correspondence reliability. Furthermore, this formulation establishes conceptual connections between alignment, attention mechanisms, and optimal transport, positioning uDTW as a principled alternative for structured matching under uncertainty.

Building on this foundation, we extend uncertainty-aware alignment beyond temporal sequences to tokenized visual representations, such as patch tokens extracted from Vision Transformers (ViTs). Visual tokens form high-dimensional sets with complex semantic structure, where identifying meaningful correspondences is  challenging and critical. We show that uDTW enables matching over such ordered token sets, and that the learned uncertainty exhibits a striking and consistent behavior: tokens corresponding to semantically relevant regions tend to have low uncertainty and dominate the alignment, while ambiguous or noisy regions are suppressed. This gives rise to a reverse-attention effect, in which uncertainty acts as a reliability-aware gating mechanism that emphasizes informative correspondences \textit{without requiring explicit attention modeling}. This property provides a new lens for understanding alignment in modern visual representations. Fig. \ref{fig:unc-attn-align} shows how uDTW-based token-to-token alignment and learned uncertainty relate (inversely) to attention.

\begin{figure*}[tbp]
\vspace{-0.5cm}
\centering
\subfloat[$\text{sDTW}_{\gamma=0.01}$\label{fig:sdtw_0.01}]{
\includegraphics[trim=2.3cm 0.258cm 2.3cm 0.258cm,clip=true,
width=0.195\linewidth,height=1.8cm]{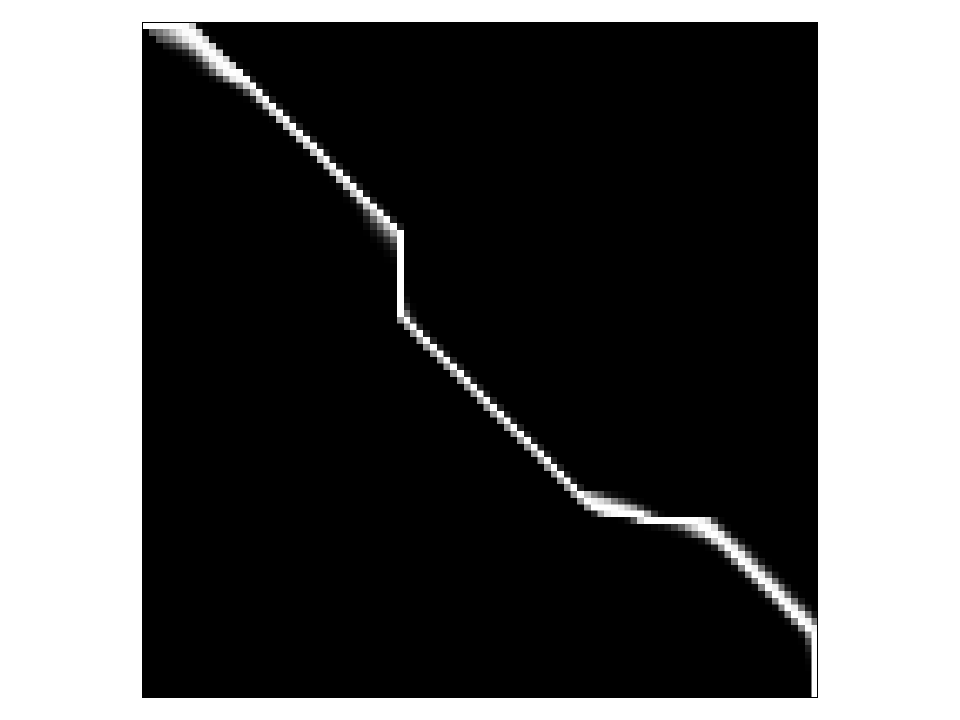}}
\subfloat[$\text{sDTW}_{\gamma=0.1}$\label{fig:sdtw_0.1}]{
\includegraphics[trim=2.3cm 0.258cm 2.3cm 0.258cm,clip=true,
width=0.195\linewidth,height=1.8cm]{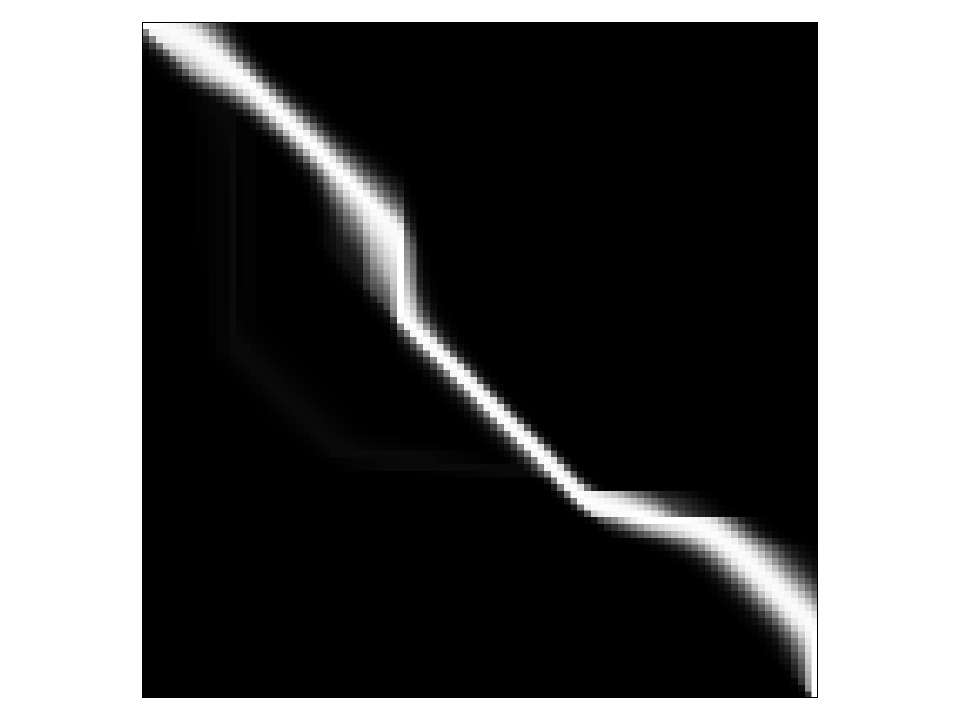}}
\subfloat[$\text{uDTW}_{\gamma=0.01}$\label{fig:udtw_0.01}]{
\includegraphics[trim=2.3cm 0.258cm 2.3cm 0.258cm,clip=true,
width=0.195\linewidth,height=1.8cm]{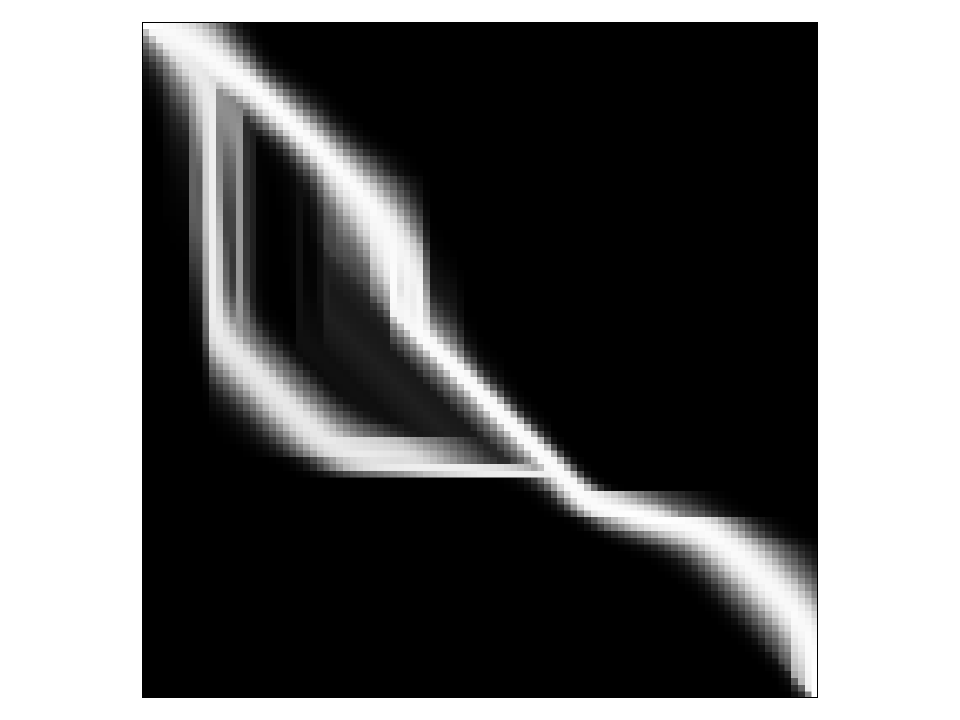}}
\subfloat[$\text{uDTW}_{\gamma=0.1}$\label{fig:udtw_0.1}]{
\includegraphics[trim=2.3cm 0.255cm 2.3cm 0.255cm,clip=true,
width=0.195\linewidth,height=1.8cm]{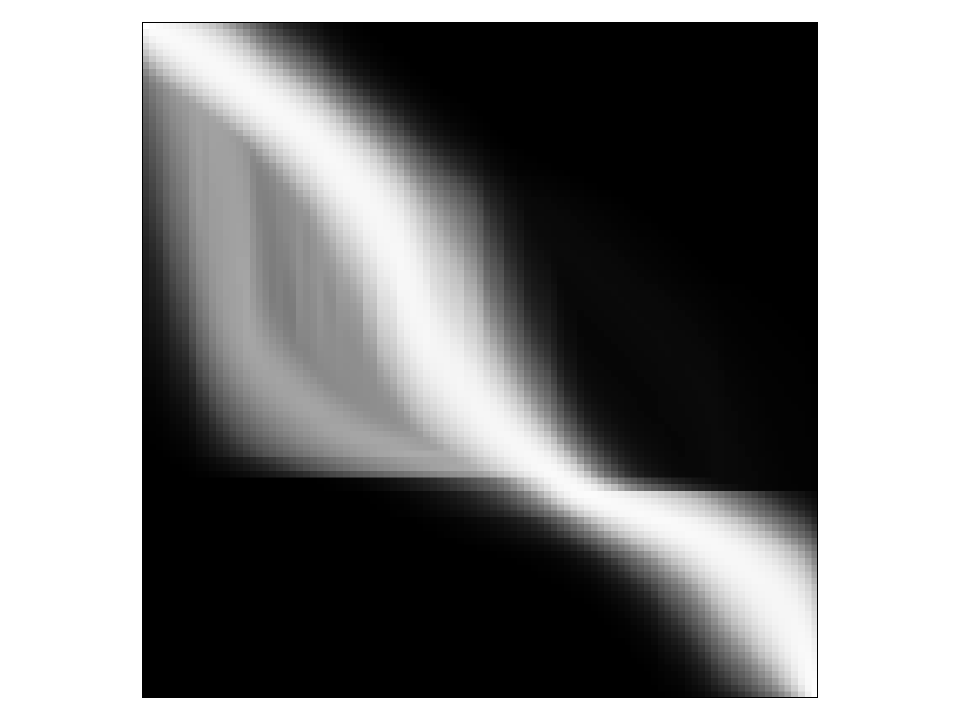}}
\subfloat[uDTW uncertainty\label{fig:udtw_cov_hiscounts}]{
\includegraphics[trim=8cm 3.48cm 8cm 3.88cm,clip=true,
width=0.195\linewidth,height=1.8cm]{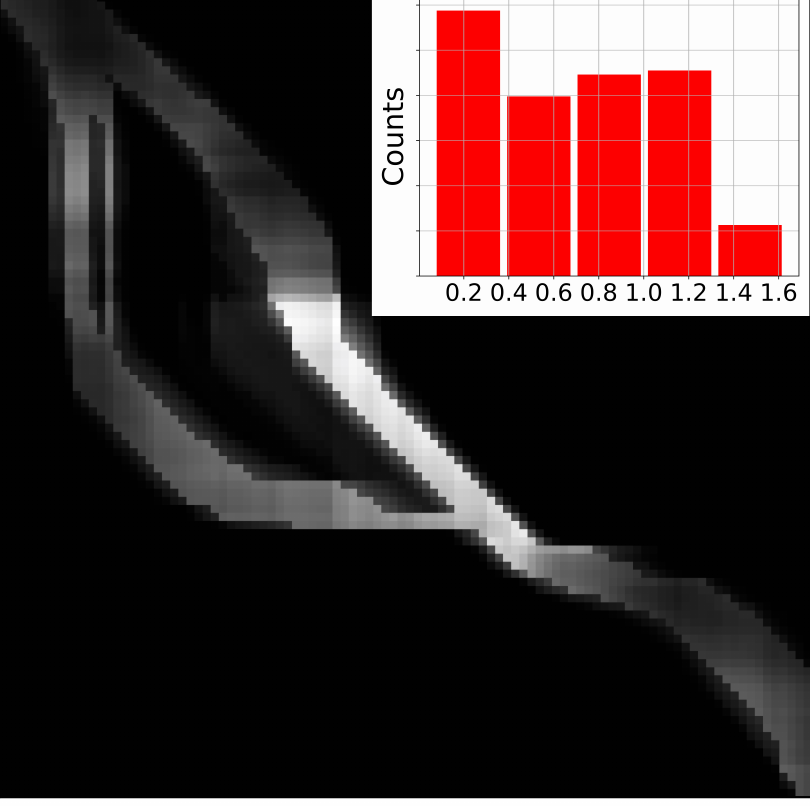}}

% \caption{Plots (a)-(d) show  paths of sDTW and uDTW (in white) for a  pair of sequences. We power-normalized pixels of plots (by the power of 0.1) to see also darker paths better.  
% With higher $\gamma$ that controls softness, in (b) \& (d) more paths  become `active' (fuzzy effect). 
% In (c), uDTW has two possible routes \vs sDTW (a) due to uncertainty modeling. 
% In (e), we visualise uncertainty $\cov$. We binarize plot (c) and multiply it by the $\cov$ to display uncertainty values on the path (white pixels = high uncertainty). The middle of the main path is deemed uncertain, which explains why an additional path merges in that region with the main path. See also the histogram of values of $\cov$.}
\caption{(a)-(d) Alignment paths of sDTW and uDTW (white) for a pair of sequences. Pixel intensities are power-normalized (exponent 0.1) to enhance low-visibility paths. Increasing the softness parameter $\gamma$ ((b),(d)) activates multiple paths, producing a smoother, distributed alignment. Compared to sDTW (a), uDTW (c) exhibits multiple plausible routes due to uncertainty-aware weighting. 
(e) Visualization of the uncertainty matrix $\cov$. The binarized path from (c) is overlaid with $\cov$ to reveal uncertainty along the alignment (white indicates high uncertainty). Elevated uncertainty in the central region explains the emergence and merging of alternative paths. The histogram summarizes the distribution of $\cov$ values.}
\label{fig:visual}
% \vspace{-0.3cm}
\end{figure*}

We validate the proposed framework across a diverse set of domains, including time series analysis, few-shot action recognition, and few-shot image classification. % on generic, fine-grained, and ultra-fine-grained benchmarks using ViT features. 
By unifying these tasks under a single alignment perspective, we demonstrate that uncertainty-aware alignment consistently improves performance while offering interpretable insights into the structure of the data. These results highlight the generality and effectiveness of modeling uncertainty in  matching problems.
This paper extends our earlier conference work \cite{wang2022uncertainty} from a sequence-specific method to a general framework for ordered data alignment. The main \textbf{contributions} are:
% A unified probabilistic alignment framework.
% Conceptual connections to major matching paradigms.
% Extension to visual token matching.
% Comprehensive evaluation across domains.
\renewcommand{\labelenumi}{\roman{enumi}.}
\begin{enumerate}[leftmargin=0.6cm]
\item We generalize uDTW into an uncertainty-aware alignment formulation that performs  matching via alignment paths with heteroscedastic uncertainty, providing robustness to noise and an interpretable likelihood-based perspective.
\item We establish links between uncertainty-aware alignment, attention mechanisms, and optimal transport, and show that uncertainty is a reliability-aware matching that differs %fundamentally 
from deterministic similarity-based approaches.
\item We extend the framework from temporal sequences to tokenized visual representations, enabling alignment over ViT features, revealing a reverse-attention effect driven by uncertainty learning.
\item We conduct extensive experiments on time series, skeleton sequences, and image classification (generic, fine-grained, and ultra-fine-grained, under supervised and unsupervised few-shot settings), demonstrating consistent improvements over state-of-the-art methods and providing new insights into uncertainty and semantic importance.
\end{enumerate}

\section{Related Work}

Below, we review the most closely related work and clearly highlight how our approach differs from prior methods.

% Existing approaches to alignment and matching are primarily based on deterministic similarity (DTW, attention), global transport (OT), or implicit robustness mechanisms. 
% In contrast, we propose a unified \emph{uncertainty-aware alignment framework} that (i) models correspondence reliability explicitly via heteroscedastic uncertainty, (ii) performs structured matching through alignment paths, and (iii) generalizes across domains from temporal sequences to visual tokens. This combination of probabilistic modeling, structural constraints, and interpretability distinguishes our work from prior approaches.

% \vspace{0.1cm}
\noindent\textbf{DTW and differentiable alignment.} DTW \cite{marco2011icml} is a classical method for aligning sequences under temporal distortions via dynamic programming. Its ability to compute optimal alignment between two sequences has made it a fundamental tool in time series analysis, speech processing, and action recognition. However, standard DTW relies on deterministic pairwise distances and selects a single minimum-cost path, making it sensitive to noise and local mismatches. 
To enable integration with modern deep learning frameworks, differentiable variant of DTW, called Soft-DTW (sDTW) \cite{marco2017icml} replaces the ``hard minimum'' calculations with a smooth approximation, enabling gradient-based optimization. Subsequent works \cite{pmlr-v130-blondel21a, dvornik2021drop, wang2022temporal, su2022temporal, wang2024meet} have explored differentiable alignment layers via implicit differentiation and declarative formulations, as well as efficient approximations for large-scale applications. Despite these advances, existing DTW-based models use  \emph{deterministic similarity calculations} as they treat features at each timestep as being equally reliable, and lack mechanism to model uncertainty in pairwise correspondences.
In contrast, our work introduces a \emph{probabilistic formulation of alignment}, where each correspondence is associated with %heteroscedastic 
uncertainty, leading to uncertainty-aware matching objective that adaptively downweights unreliable features to form robust alignment.

% \vspace{0.1cm}
\noindent\textbf{Attention- and learning-based correspondence.} Attention mechanisms \cite{dosovitskiy2020image,rao2021dynamicvit, chen2024motion} are the dominant paradigm for learning correspondences in sequences and visual representations. 
By computing similarity-based weights (\eg, via softmax), attention enables flexible, dense matching between elements and has been widely adopted in transformers and related architectures. Variants have also been explored for sequence alignment, integrating attention into warping or matching processes \cite{cai2019dtwnet, iwana2020dtw, matsuo2023deep}.
However, attention mechanisms are inherently \emph{similarity-centric}: they assign weights based on relative similarity scores without explicitly modeling the reliability of features. As a result, attention may amplify spurious correspondences when similarity estimates are noisy or ambiguous.
Our approach differs fundamentally in both structure and mechanism. Firstly, instead of dense matching, we perform \emph{structured alignment via paths}, preserving matching order constraints. Secondly, uncertainty is explicitly modeled to directly model matching costs. This results in a distinct behavior: unreliable correspondences are suppressed rather than amplified, giving rise to a \emph{reverse-attention effect}, where semantically relevant regions dominate the alignment. %This establishes uncertainty-aware alignment as a principled alternative to similarity-based attention.

% \vspace{0.1cm}
\noindent\textbf{Optimal transport and structured matching.} Optimal Transport (OT) provides a powerful framework for matching distributions by computing a global transport plan between two distributions \cite{montesuma2024recent}. Regularized OT, \eg, Sinkhorn \cite{cuturi2013sinkhorn} enable scalable and differentiable matching in a variety of vision and learning problems. Recent works \cite{lee2019hierarchical, janati2020spatio} have explored connections between OT and sequence alignment, including hybrid formulations that combine OT with temporal constraints.
A key distinction lies in the structure of the matching. OT computes \emph{global correspondences} without enforcing strict ordering constraints, whereas DTW enforces that matching indexes between two sequences always progress forward (or sometimes stay still) in each matching step, albeit at a variable rate. Thus, generic OT ignores this structural constraint inherent to sequentially ordered data.
% 
%Our method follows such \emph{path constraints under uncertainty-aware matching}. 
In contrast, uDTW performs structured matching under monotonic constraints on indexes while incorporating heteroscedastic uncertainty. % into the alignment objective. %This bridges sequence alignment and distribution matching, providing a unified perspective that combines structural consistency with probabilistic modeling of correspondence reliability.

% \vspace{0.1cm}
\noindent\textbf{Robust alignment and uncertainty modeling.} Robust alignment under noise and variability has been widely studied. Classical approaches introduce constraints such as warping windows or global penalties to regularize alignments \cite{salvador2007toward, senin2008dynamic, marteau2008time}. More recent methods \cite{hoffer2015deep, cai2019dtwnet, matsuo2023deep} explore parametric warping functions, diffeomorphic transformations, or learned similarity metrics to improve robustness.
Despite these efforts, robustness is typically achieved \emph{implicitly} through model design rather than by explicitly modeling uncertainty. Consequently, these methods do not provide a principled measure of matching reliability and remain sensitive to heteroscedastic noise.
In contrast, our approach explicitly models \emph{aleatoric uncertainty} at the level of pairwise correspondences. This results in a maximum-likelihood formulation with a precision-weighted matching term and a log-variance regularization term, yielding both robustness and interpretability. Uncertainty directly reflects the reliability of features and plays a central role in shaping the alignment.

% \vspace{0.1cm}
\noindent
\textbf{Alignment for visual tokens.} With the emergence of ViTs \cite{dosovitskiy2020image, rao2021dynamicvit}, visual data is increasingly represented as sets of tokens. Most existing approaches compare such representations using global pooling, similarity aggregation, or attention-based matching \cite{sarlin2020superglue, xiao2023token, jung2024scale,cao2024madtp}, often without explicit structural constraints.
We extend alignment from temporal sequences to \emph{tokenized visual representations}, treating tokens as structured entities amenable to alignment. Our framework enables path-based matching over ordered token sets while using uncertainty to improve correspondences. Notably, the learned uncertainty exhibits a consistent and interpretable pattern: semantically relevant regions are assigned low uncertainty and dominate the alignment, while ambiguous regions are suppressed, enabling a new perspective linking alignment, uncertainty, and semantic importance.

\section{Method: Uncertainty-Aware Alignment}
\label{sec:method}

\subsection{Problem Formulation}
\label{sec:problem_formulation}

%\vspace{0.1cm}
\noindent
\textbf{Notations.} We consider the problem of aligning ordered sets of observations. % under uncertainty. 
Let $\mX = [\vx_1, \vx_2, \dots, \vx_\tau] \in \mbr{d \times \tau}$ and 
$\mX' = [\vx'_1, \vx'_2, \dots, \vx'_{\tau'}] \in \mbr{d \times \tau'}$ 
be two collections of lengths $\tau$ and $\tau'$, 
where $\vx_i, \vx'_j \in \mbr{d}$ are feature vectors or tokens extracted from time series, skeleton joints, or visual patch embeddings. %Concatenations of vectors and matrices follow the conventions: $[\vx_i]_{i \in \idx{I}}$ for vectors, $[x_{ij}]_{(i,j) \in \idx{I}\times \idx{J}}$ for matrices.
% For notational convenience, we represent both sequences and sets as ordered collections, while the alignment constraints will determine whether ordering is enforced.

A path $\mPi$ is defined as a sequence of index pairs $(m,n)$ satisfying structural constraints, 
and $\mathcal{A}_{\tau,\tau'}$ denotes the set of all admissible alignment paths with cardinality of the Delannoy number $D(\tau-1,\tau'-1)
=
\sum_{i=0}^{\min(\tau-1,\tau'-1)}
\binom{\tau-1}{i}\binom{\tau'-1}{i}2^i$. 
Its corresponding matrix form is denoted by $\mathbf{\Pi} \in %\mathbb{R}
[0, 1]^{\tau \times \tau'}$, 
where $\Pi_{mn}$ indicates the contribution of aligning $\vx_m$ with $\vx'_n$. 
In the classical (hard) setting, $\mathbf{\Pi}$ is a sparse binary matrix encoding a valid path, 
whereas in the soft-relaxed formulation it is denser. %, 
%induced by a Gibbs distribution over alignment paths.
%
The alignment task seeks an optimal alignment plan that minimizes a pairwise distance measure. %, optionally accounting for observation uncertainty. 
Let $\mD \!\equiv\![d^2_{mn}]_{(m,n)\in\idx{\tau}\!\times\!\idx{\tau'}} \!\in\! \mbrp{\tau \times \tau'}$ be the pairwise distance matrix with entries $d^2_{mn} \!=\! \|\vx_m \!-\! \vx'_n\|^2_2$. %(or other suitable dissimilarity measures). 
A classical DTW then seeks:
\begin{align}
    \label{eq:classical_alignment}
    \mathbf{\Pi}^\ast 
    = \argmin_{\mathbf{\Pi} \in \mathcal{A}_{\tau,\tau'}} 
    \langle \mathbf{\Pi}, \mD \rangle,
\end{align}
where %$\mathcal{A}_{\tau,\tau'}$ denotes the set of admissible alignment plans (\eg, monotonic paths for time series or matching constraints for ordered tokens), and 
$\langle \cdot, \cdot \rangle$ denotes the matrix inner product, defined as 
$\langle \mPi, \mD \rangle \equiv \langle \mathrm{vec}(\mPi), \mathrm{vec}(\mD) \rangle$. 
% In practice, we will later adopt a soft-relaxed formulation that replaces hard constraints with a dense matching distribution over pairwise correspondences.

%\vspace{0.1cm}
\noindent
\textbf{Motivation for probabilistic alignment.} 
In practical applications, feature observations are often corrupted by heteroscedastic noise, missing data, or ambiguous correspondences 
(\eg, occluded joints, variable-speed actions, or visual patch ambiguities). Deterministic alignment in Eq.~\eqref{eq:classical_alignment} 
cannot capture such uncertainty, often leading to suboptimal matches.
To address this limitation, we introduce a \textit{probabilistic alignment framework}, where each pairwise correspondence $(m,n)$ 
is associated with a variance $\sigma^2_{mn} > 0$ quantifying its reliability. Collecting these variances defines an \emph{uncertainty matrix} $\cov \equiv [\sigma^2_{mn}]_{(m,n) \in \idx{\tau} \times \idx{\tau'}} \in \mathbb{R}_+^{\tau \times \tau'}$.
We further define its element-wise inverse (precision) as
$\cov^\dagger \equiv [\sigma_{mn}^{-2}]_{(m,n) \in \idx{\tau} \times \idx{\tau'}}$.

Under this formulation, correspondences with higher uncertainty contribute less to the alignment cost, reducing the impact of noise. 
Formally, the weighted structured alignment problem becomes:
\begin{align}
\label{eq:weighted_alignment}
\mathbf{\Pi}^\ast
= \!\!\argmin_{\cov^\dagger,\,\mathbf{\Pi} \in \mathcal{A}_{\tau,\tau'}} 
\!\!\Big\langle \mathbf{\Pi}, \mD \odot \cov^\dagger \Big\rangle
+ \beta \Big\langle \mathbf{\Pi}, \log \cov \Big\rangle,
\end{align}
where $\odot$ denotes element-wise multiplication, and $\beta \ge 0$ controls the strength of the uncertainty regularization.

The first term represents a \emph{precision-weighted matching cost}, down-weighting unreliable correspondences, while the second term penalizes large uncertainty estimates. % , preventing degenerate solutions. 
% When all $\sigma^2_{mn} = 1$, Eq.~\eqref{eq:weighted_alignment} recovers the classical alignment objective in Eq.~\eqref{eq:classical_alignment}.
% This probabilistic formulation generalizes classical DTW to an uncertainty-aware alignment framework that applies to both sequences and token-based representations. 
% It provides: 
% (i) robustness to noise and outliers via uncertainty-aware weighting, 
% (ii) a principled foundation for differentiable relaxation through soft-min or path probability distributions, and 
% (iii) a unified treatment of structured inputs across domains (\eg, time series, skeletons, and visual tokens).
In the following, we derive a principled \textit{uncertainty-aware alignment operator (uDTW)} from Eq.~\eqref{eq:weighted_alignment}, along with its probabilistic interpretation and properties. %, and extensions. % to token-based alignment.

\subsection{Probabilistic Alignment with Heteroscedastic Uncertainty}
\label{sec:prob_alignment}

%\vspace{0.1cm}
\noindent
\textbf{Maximum likelihood path estimation.} Building on the general alignment formulation in Eq.~\eqref{eq:weighted_alignment}, we model each correspondence $(m,n)$ as a random variable with Gaussian uncertainty:
$p(\vx_m \mid \vx'_n, \sigma^2_{mn}) = \mathcal{N}(\vx_m; \vx'_n, \sigma^2_{mn})$,
where $\sigma^2_{mn} > 0$ captures the heteroscedastic noise associated with matching $\vx_m$ to $\vx'_n$.  
Consider an alignment path $\mPi_i$. % \subseteq \idx{\tau} \times \idx{\tau'}$. 
Under the assumption of independence, the likelihood of the path is
\begin{align}
\mathcal{L}(\mPi_i) &\!=\!\! \prod_{(m,n)\in \mPi_i} p(\vx_m \mid \vx'_n, \sigma^2_{mn}) \nonumber \\
&\!=\!\! \prod_{(m,n)\in \mPi_i} \frac{1}{(2\pi \sigma^2_{mn})^{d/2}} \exp\Big(-\frac{\|\vx_m - \vx'_n\|_2^2}{2 \sigma^2_{mn}}\Big).
\label{eq:mle}
\end{align}

Maximizing the above likelihood is equivalent to minimizing log-likelihood:
\begin{align}
\label{eq:weighted_path}
\!\!\!\!\cov_i^\ast &= \min_{\{\sigma_{mn}\}_{(m,n)\in\mPi_i}} 
\sum_{(m,n)\in \mPi_i} \frac{\|\vx_m - \vx'_n\|_2^2}{\sigma^2_{mn}} + d \log \sigma^2_{mn}.
\end{align}
where the first term penalizes large distances scaled by precision, and the second term regularizes against overestimating uncertainty. $\cov_i^\ast$ contains variances for path $\mPi_i$. Eq.~\eqref{eq:weighted_path} generalizes classical DTW by incorporating heteroscedastic uncertainty. If all $\sigma^2_{mn} = 1$, it reduces to standard DTW. When $\sigma^2_{mn}>0$ vary across correspondences, the alignment becomes uncertainty-aware.

%\vspace{0.1cm}
\noindent
\textbf{Relation to uDTW.} To connect the matrix-based formulation with individual alignment path $\mPi_i$ with its path-specific cost directly derived from the objective in Eq.~\eqref{eq:weighted_alignment}:
\begin{align}
\label{eq:uDTW_terms}
\left\{
\begin{array}{l}
d^2_{\mPi_i} \equiv \Big\langle \mPi_i, \mD \odot \cov^\dagger \Big\rangle \\
\Omega_{\mPi_i} \equiv \Big\langle \mPi_i, \log \cov \Big\rangle.
\end{array}
\right.
\end{align}
Here, $d^2_{\mPi_i}$ represents the precision-weighted matching cost along the path, and $\Omega_{\mPi_i}$ is the regularization term penalizing large uncertainty estimates on the path.  
The  alignment path $\mPi_i$ under this probabilistic framework has variance $\cov_i^\ast$:
\begin{align}
\label{eq:uDTW_path}
\cov_i^\ast &= \min_{\cov} 
d^2_{\mPi_i} + \beta \Omega_{\mPi_i}, \quad \beta \ge 0,
\end{align}
which defines the uDTW objective for $\mPi_i$.  
This formulation explicitly shows %that minimizing the uncertainty-weighted alignment in Eq.~\eqref{eq:weighted_alignment} is equivalent to maximizing the likelihood of a path in path-space (
connection with Eq.~\eqref{eq:mle}. % Consequently, uDTW provides a principled uncertainty-aware generalization of classical DTW (Eq. \eqref{eq:weighted_path}), naturally bridging matrix-level formulations with path-level optimization. %, and supporting differentiable, end-to-end learning for structured sequences and token-based representations.

%\vspace{0.1cm}
\noindent
\textbf{Parameterizing uncertainty.}  
Directly optimizing the full uncertainty matrix $\cov$ is impractical for variable-length sequences. % ($\tau \!\!\neq\!\! \tau'$). 
Instead, we predict uncertainties using a small neural network $\sigma(\cdot; \mathcal{P}_\sigma)$, called \emph{SigmaNet}. 
In practice, each pairwise uncertainty $\sigma_{mn}^2$ can be computed by either (i)  additive per-token variances $\sigma_{mn}^2 \!\!=\!\! 0.5[\sigma^2(\vx_m) \!\!+\!\! \sigma^2(\vx'_n)]$ or (ii) a jointly predicted pairwise variance $\sigma_{mn}^2 \!\!=\!\! \sigma^2(\vx_m, \vx'_n)$. These predicted uncertainties are then used to weight the alignment cost in Eq. \eqref{eq:weighted_path} %\eqref{eq:uDTW_path})
for end-to-end learning of both feature encoder and uncertainty predictor. 

% \textbf{Parameterizing uncertainty.}  
% Direct optimization over $\cov$ is impractical for variable-length sequences ($\tau \neq \tau'$). We instead predict uncertainties with a small neural network $\sigma(\cdot;\mathcal{P}_\sigma)$, called \emph{SigmaNet}:
% \begin{align}
% \label{eq:additive_cov_revised}
% \cov &= 0.5  \big[ \sigma^2(\vx_m;\mathcal{P}_\sigma) + \sigma^2(\vx'_n;\mathcal{P}_\sigma) \big]_{(m,n) \in \idx{\tau} \times \idx{\tau'}}, \\
% \label{eq:joint_cov_revised}
% \cov' &= [\sigma^2(\vx_m, \vx'_n;\mathcal{P}_\sigma)]_{(m,n) \in \idx{\tau} \times \idx{\tau'}}.
% \end{align}
% Eq.~\eqref{eq:additive_cov_revised} uses additive per-token variances, while Eq.~\eqref{eq:joint_cov_revised} allows modeling the pairwise uncertainty jointly. These predictions enable end-to-end learning of both the feature encoder and uncertainty network.

%\vspace{0.1cm}
\noindent
\textbf{Soft relaxation.}  
To facilitate differentiable training, soft uDTW  aggregates over all admissible alignment paths: % Let $\mD$ be the pairwise distance matrix, and $\cov$ the corresponding uncertainty matrix. 
%We define:
\begin{align}
&\!\!\!\!\left\{\!\!
\begin{array}{l}
{d^2_{\text{uDTW}}% (\mD,\cov^{\dag})
}
\!=\!\softming\!\Big(\underbrace{\left[\left\langle\mPi,\mD\!\odot\!\cov^{\dag}\right\rangle\right]_{\mPi\in\tAnb_{\tau,\tau'}}}_{\vw}\Big),\!
\\ 
\Omega% ({\cov})
\!=\!\softminsel\!\left(\vw, 
\left[\left\langle\mPi,\log\!\cov\right\rangle\right]_{\mPi\in\tAnb_{\tau,\tau'}}
\right),
\end{array}
\right.
\label{eq:soft_uDTW_compact}
\end{align}
% where
\begin{align}
&\!\!\!\!\text{where} \left\{\!\!
\begin{array}{l}
\label{eq:selectors}
    \!\softming(\mathbf{w}) \!=\! \sum_i w_i \frac{\exp(-(w_i - \mu)/\gamma)}{\sum_j \exp(-(w_j - \mu)/\gamma)}, \\
    \!\softminsel(\mathbf{w}, \vu) \!=\! \sum_i u_i \frac{\exp(-(w_i - \mu)/\gamma)}{\sum_j \exp(-(w_j - \mu)/\gamma)},
\end{array}
\right.
\end{align}
with $\mu$ the mean of $\mathbf{w}$ for numerical stability and $\gamma>0$ controlling the smoothness. The operator $\softminsel$ acts as a soft selector, returning the value associated with the minimal path when $\gamma \!\to\! 0$. Both $d^2_{\text{uDTW}}(\mD,\cov^\dag)
$ and $\Omega(\cov)$ depend on the input sequences, making them fully differentiable \wrt the token embeddings. The vector $\mathbf{w}$ contains the path-aggregated precision-weighted distances for all admissible paths, and $\Omega% (\cov)
$ is the aggregation of log-variance penalties along these paths.

This formulation decouples the precision-weighted path distance $d^2_{\text{uDTW}}$ from the uncertainty regularization $\Omega% (\cov)
$, yielding a fully differentiable, uncertainty-aware alignment. During training, $\Omega% (\cov)
$ penalizes the model for extreme uncertainty, while $d^2_{\text{uDTW}}$ captures the smallest matching distance between an two sets aligned. When $\gamma \to 0$, $\Omega% (\cov)
$ reduces to the log-variance along the minimal-cost path. %, recovering the classical uDTW behavior (Eq. \eqref{eq:uDTW_terms}).

% This formulation naturally extends uDTW to:  
% (i) weighted structured matching under heteroscedastic uncertainty,  
% (ii) differentiable alignment for end-to-end learning in deep networks, and (iii) domain-agnostic applicability, including time series, structured sequences, and token-based visual features.

% We next extend this to set-to-set alignment, enabling uDTW to operate over visual token embeddings (\eg, ViT patches) while retaining all probabilistic guarantees.

\subsection{uDTW Properties and Theoretical Insights}
\label{sec:udtw_theory}

The probabilistic alignment formulation of uDTW % (Sec.~\ref{sec:prob_alignment}) 
is not only practical but also admits several desirable properties. %, which distinguish it from classical DTW, soft-DTW (sDTW), and other alignment-based methods. 
Fig. \ref{fig:visual} shows how uncertainty affects uDTW. % We summarize the key properties below.

% Formally, consider a perturbed sequence $\vx'_n + \epsilon_n$ with zero-mean noise $\epsilon_n \sim \mathcal{N}(0, \Sigma_n)$. The expected path cost along $\mPi_i$ becomes:
% \begin{align}
%     \mathbb{E}[d^2_{\mPi_i}] &= \sum_{(m,n)\in\mPi_i} \frac{\|\vx_m - (\vx'_n + \epsilon_n)\|_2^2}{\sigma^2_{mn}} \nonumber \\ 
%     & = \sum_{(m,n)\in\mPi_i} \frac{\|\vx_m - \vx'_n\|_2^2 + \text{tr}(\Sigma_n)}{\sigma^2_{mn}},
% \end{align}
% which shows that the additional noise contributes additively but is scaled by $\sigma^2_{mn}$. Therefore, higher uncertainty estimates naturally reduce the impact of noisy frames, providing robustness that classical DTW lacks.

%\vspace{0.1cm}
\noindent
\textbf{Robustness to noise.}  
By modeling heteroscedastic uncertainty $\sigma^2_{mn}$, uDTW down-weights contributions from noisy or unreliable tokens.

\begin{proposition}[Alignment under additive noise]
Let $[\vx_1+\boldsymbol{\epsilon}_1, \dots, \vx_\tau+\boldsymbol{\epsilon}_\tau]$ and $[\vx'_1+\boldsymbol{\epsilon}'_1, \dots, \vx'_{\tau'}+\boldsymbol{\epsilon}'_{\tau'}]$ be sequences with corresponding additive i.i.d. Gaussian noises $\boldsymbol{\epsilon}_m\!\sim\!\mathcal{N}(0, \gamma^2_m)$ and $\boldsymbol{\epsilon}'_n\!\sim\!\mathcal{N}(0, \gamma'^2_n)$. Then, for a given alignment path $\mPi_i$:
% \begin{align}
$\mathbb{E}_{\epsilon,\epsilon'}[d^2_{\mPi_i}] = \sum_{(m,n)\in\mPi_i} 
%\frac{\|\vx_m - \vx'_n\|_2^2 + \mathrm{tr}(\Sigma_n)}{\sigma^2_{mn}}
\frac{\|\vx_m - \vx'_n\|_2^2}{\sigma^2_{mn}}
+
d\sum_{(m,n)\in\mPi_i}
\frac{\gamma_m^2+{\gamma'_n}^2}{\sigma^2_{mn}}
$.
% \end{align}

\noindent
The uncertainty-weighted distance can reduce noises $\gamma_m^2+{\gamma'_n}^2$ and the noisy samples in the numerator by dividing them by $\sigma^2_{mn}$, e.g., additively factorized $\sigma^2_{mn}\!=\!0.5\big[\sigma^2(\vx_m)+\sigma^2(\vx'_n)\big]$.
\end{proposition}

\begin{proof}% [Proof sketch]
Expanding the squared norm and taking expectation:
$\mathbb{E}_{\epsilon_m,\epsilon'_n}\big[\|\vx_m\!+\!\boldsymbol{\epsilon}_m \!-\! (\vx'_n\!+\!\boldsymbol{\epsilon}'_n)\|_2^2\big]
\!=\!  \mathbb{E}_{\epsilon^\ast_{mn}}\big[\|\mathbf{y}_{mn}+\boldsymbol{\epsilon}^\ast_{mn}\|^2_2\big]\!=\!\|\mathbf{y}_{mn}\|^2_2+\mathbb{E}_{\epsilon^\ast_{mn}}\big[\|\boldsymbol{\epsilon}^\ast_{mn}\|^2_2\big]$, 
where $\mathbf{y}_{mn}\!=\!\vx_m\!-\!\vx'_n,\,\boldsymbol{\epsilon}^\ast_{mn}\!=\!\boldsymbol{\epsilon}_m\!-\!\boldsymbol{\epsilon}'_n$ and $\boldsymbol{\epsilon}^\ast_{mn}\!\sim\!\mathcal{N}(0, \gamma^2_m+\gamma'^2_n)$.
%
%
% \begin{align}
% & \mathbb{E}\big[\|\vx_m - (\vx'_n + \epsilon_n)\|_2^2\big] \nonumber\\ 
% = & \mathbb{E}\big[\|\vx_m - \vx'_n\|_2^2 - 2 (\vx_m - \vx'_n)^\top \epsilon_n + \|\epsilon_n\|_2^2\big] \nonumber\\
% = & \|\vx_m - \vx'_n\|_2^2 - 2 (\vx_m - \vx'_n)^\top \mathbb{E}[\epsilon_n] + \mathbb{E}[\|\epsilon_n\|_2^2] \nonumber\\
% = & \|\vx_m - \vx'_n\|_2^2 + \text{tr}(\Sigma_n), \nonumber
% \end{align}
%where we used $\mathbb{E}[\epsilon_n] = 0$ and $\mathbb{E}[\|\epsilon_n\|_2^2] = \mathrm{tr}(\Sigma_n)$.  
Summing over $(m,n) \in \mPi_i$ and applying the uncertainty weighting $\sigma^2_{mn}$ gives the expected uDTW cost.
\end{proof}

\begin{remark}
uDTW generalizes sDTW by explicitly accounting for the observation noise, providing a principled mechanism for robust sequence alignment.
\end{remark}

%{\color{red}here}
%\vspace{0.1cm}
\noindent
\textbf{Relation to attention \& path-constrained transport.}  %$\mathcal{A}_{\tau,\tau'}$ 
Let $\mathcal{A}_{\tau,\tau'}\!\!=\!\!\bigl\{
(m_1,n_1),\dots,(m_L,n_L)\!\!:\!\!
(m_1,n_1)\!\!=\!\!(1,1),
(m_L,n_L)\!\!=\!\!(\tau,\tau'),\,
(m_{\ell+1}\!-\!m_\ell,\; n_{\ell+1}\!-\!n_\ell)\in\big\{(1,0),(0,1),(1,1)\big\},\,L\!\geq\!1
\bigr\}$ be the set of all admissible alignment paths of lengths $L$. 
Each path $\mPi_i \in \mathcal{A}_{\tau,\tau'}$ can also be represented as a binary matrix $\mPi_i \in \{0,1\}^{\tau \times \tau'}$ satisfying monotonicity constraints $(m_{\ell+1}\!-\!m_\ell, n_{\ell+1}\!-\!n_\ell)\!\in\!\big\{(1,0),(0,1),(1,1)\big\}$ for $l\!=\!1,\ldots,L$.

For each path $\mPi_i$, define the precision-weighted path cost $w_i \equiv \langle \mPi_i, \mD \odot \cov^\dag \rangle$.
The soft uDTW distance is given by $d^2_{\text{uDTW}} = \softming(\mathbf{w}) = \sum_i w_i \cdot \pi(\mPi_i)$,
where $\pi(\mPi_i) = \frac{\exp\big(-(w_i - \mu)/\gamma\big)}{\sum_{j:\mPi_j \in \mathcal{A}_{\tau,\tau'}} \exp\big(-(w_j - \mu)/\gamma\big)}$ defines a Probability Mass Function (PMF) over all admissible alignment paths, \ie, $\sum_{i=1}^{D(\tau-1,\tau'-1)} \pi(\mPi_i)=1$. 

The same weighting $\pi(\mPi)$ is used in $\softminsel$ in Eq. \eqref{eq:selectors}, ensuring that both the alignment cost and the uncertainty regularization are aggregated under the PMF of all admissible paths.
This weighting induces a \textit{pairwise soft-coupling matrix} $\mPi^{\star}\!=\! [\pi^\star_{mn}]_{(m,n)\in\idx{\tau}\times\idx{\tau'}}$, defined as the  probability that elements $\vx_m$ and $\vx'_n$ are aligned:
$\pi^\star_{mn}\!=\!\sum_{\mPi \in \mathcal{A}_{\tau,\tau'} : [\mPi_i]_{mn}=1} \pi(\mPi)$.

\begin{theorem}[Equivalence of path aggregation and pairwise coupling aggregation]
Under the above definitions, the soft uDTW distance admits the equivalent representation
$d^2_{\text{uDTW}}(\mathbf{X},\mathbf{X}') = \sum_{m=1}^{\tau}\sum_{n=1}^{\tau'} \pi^\star_{mn} \frac{\|\vx_m - \vx'_n\|_2^2}{\sigma^2_{mn}}$.  
\end{theorem}

\begin{proof}
    Substituting $w_i\!=\!\sum_{m,n} [\mPi^\star_i]_{mn} \frac{\|\vx_m - \vx'_n\|_2^2}{\sigma^2_{mn}}$ into $%d^2_{\text{uDTW}}= % = \softming(\mathbf{w}) 
    \sum_i w_i \cdot \pi(\mPi_i)$ and exchanging the order of summation yields
$d^2_{\text{uDTW}} 
= \sum_{m,n} \Big(\underbrace{\sum\nolimits_i \pi(\mPi_i)[\mPi^\star_i]_{mn}}_{=\pi^\star_{mn}} \Big) \frac{\|\vx_m - \vx'_n\|_2^2}{\sigma^2_{mn}}$ where 
the inner summation equals $\pi^\star_{mn}$ by definition, completing the proof.
\end{proof}

The coupling $\pi^*_{mn}$ can be interpreted as a \textit{structured attention map}. Unlike conventional attention, which normalizes local similarity scores, each $\pi^*_{mn}$ aggregates contributions over all admissible alignment paths,  reserving matching order
constraint. Furthermore, the precision weighting $\cov^\dag$ modulates contributions across all paths, ensuring that uncertain correspondences are systematically down-weighted.

The same coupling induces a transport-like objective in which $\pi^*_{mn}$ acts as a soft assignment between the two sequences. In contrast to classical optimal transport, the coupling here is restricted to those induced by distributions over admissible alignment paths, thereby encoding structural constraints such as monotonicity and continuity.

% {\color{red}here}
%\vspace{0.1cm}
\noindent\textbf{Connection to probabilistic alignment.} The soft uDTW formulation admits a probabilistic interpretation in which alignment is modeled as a latent variable distributed over the space of admissible paths.
The soft uDTW distance corresponds to the expected alignment cost 
$d^2_{\text{uDTW}}$. % = \mathbb{E}_{\pi}[v_i]$.
Unlike deterministic alignment methods that select a single optimal path, our formulation maintains a distribution over all feasible paths, allowing the model to represent ambiguity in the alignment.
A key property of this formulation is that the same path distribution $\pi(\mPi_i)$ is used consistently across both the alignment objective and the uncertainty regularization. In particular, the uncertainty term
$\Omega = \softminsel(\mathbf{w}, \mathbf{u}) = \sum_i u_i \cdot \pi(\mPi_i)$,
aggregates path-wise log-variance quantities $u_i = \langle \mPi_i, \log \cov \rangle$ under the same distribution. This shared weighting ensures that the model jointly evaluates alignment quality and uncertainty under a common probabilistic structure, enabling coherent propagation of uncertainty through the alignment process.
The temperature parameter $\gamma > 0$ controls the concentration of the path distribution. This behavior is formalized below.

\begin{theorem}[Deterministic limit of uDTW]
    Let $\pi(\mPi_i)$ be defined as above, and let $\mPi^* = \arg\min_{\mPi_i \!\in\! \mathcal{A}_{\tau,\tau'}} w_i$ denote an optimal alignment path. Then, as $\gamma \to 0$, the distribution $\pi(\mPi_i)$ converges to a Dirac distribution concentrated on $\mPi^*$, and $d^2_{\text{uDTW}} \!\to\! w^*$, $\Omega \!\to\! u^*$, where $w^* \!=\! w(\mPi^*)$ and $u^* \!=\! u(\mPi^*)$.
\end{theorem}

\begin{proof}
For any $\mPi_i \neq \mPi^*$, we have $w_i > w^*$. Hence $\frac{\pi(\mPi_i)}{\pi(\mPi^*)}
= \exp\!\left(-\frac{w_i - w^*}{\gamma}\right) \to 0$ as $\gamma \!\to\! 0$.
Thus $\pi(\mPi^*) \to 1$ and $\pi(\mPi_i) \to 0$ for all $\mPi_i \neq \mPi^*$. The convergence of $d^2_{\text{uDTW}}$ and $\Omega$ follows directly from their definitions as expectations under $\pi$.    
\end{proof}

This result shows that uDTW provides a continuous relaxation of classical alignment, interpolating between deterministic path selection and distributed alignment over multiple paths. For finite $\gamma$, the model captures uncertainty in the alignment itself, while for $\gamma \to 0$, it recovers the classical uDTW formulation. Crucially, the incorporation of heteroscedastic precision $\cov^\dag$ ensures that this uncertainty is data-dependent, allowing unreliable correspondences to be systematically attenuated across all plausible alignments.

This probabilistic alignment view provides a formal latent-variable interpretation of the path distribution underlying the structured coupling presented previously.

\begin{figure*}[tbp]
\vspace{-0.2cm}
  \centering
  \includegraphics[width=\linewidth]{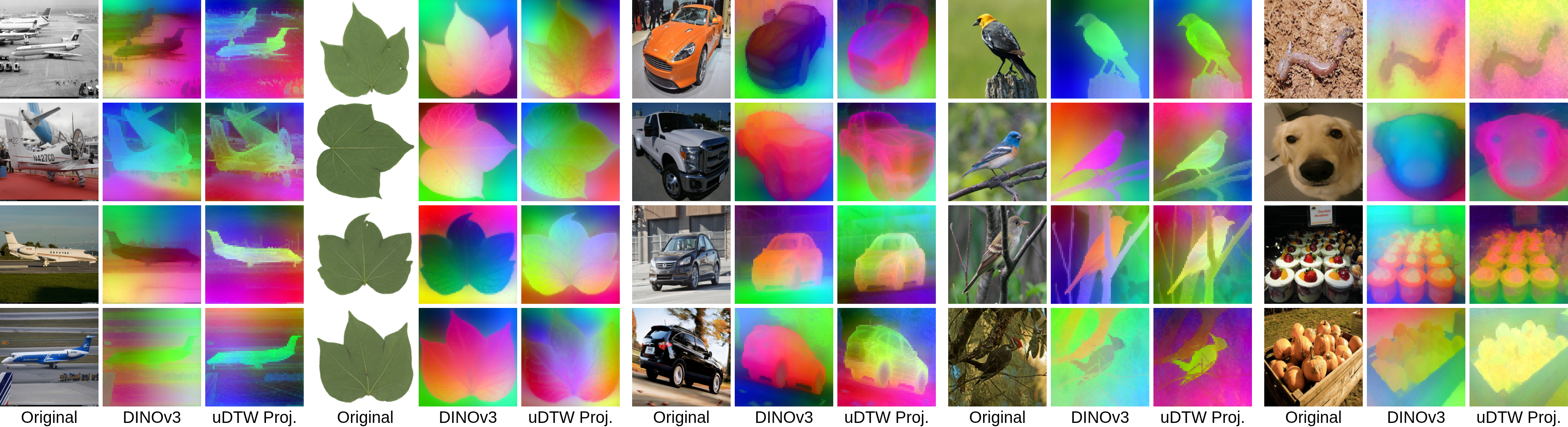}
  \vspace{-0.6cm}
  \caption{
Visualization of visual-token embeddings before and after the learned uDTW projection. Token embeddings are visualized by mapping the first three PCA components into RGB channels. Compared to the original DINOv3 embeddings, the uDTW-projected embeddings become more compact and object-centric, with clearer foreground-background separation and stronger structural continuity. On ultra-fine-grained leaf images, the projected embeddings preserve subtle structures such as contours, defects, and vein patterns, suggesting that uDTW induces a reliability-aware semantic token geometry.
}
  \label{fig:dinov3-udtw}
  %\vspace{-0.3cm}
\end{figure*}

\subsection{Extension to Tokenized Visual Representations}
\label{sec:visual_tokens}

Although uDTW was originally formulated for temporal sequences, its probabilistic, uncertainty-aware framework extends naturally to tokenized visual representations, such as patch embeddings extracted from ViTs. Each token represents a local image patch (\eg, $16 \!\times\! 16$ pixels) and may include positional encoding. Visual tokens form a structured 2D grid, but weak ordering does not impede uDTW’s structured alignment. % In this setting, heteroscedastic uncertainty provides both robustness and interpretability: tokens corresponding to semantically meaningful regions tend to have low uncertainty and dominate the alignment, whereas ambiguous or background tokens exhibit higher uncertainty and are systematically down-weighted. 
Unlike conventional attention mechanisms, % which normalize similarity scores across tokens, 
uDTW incorporates uncertainty directly into the alignment cost, producing a reliability-aware mechanism that emphasizes informative tokens without requiring an explicit attention module.

%\vspace{0.1cm}
\noindent
\textbf{Uncertainty-driven reverse attention.} Let $[\vx_m]_{m \in \idx{T}}$ and $[\vx'_n]_{n \in \idx{T}}$ denote the query and support token sets, arranged in a $\sqrt{T} \times \sqrt{T}$ token grid. 
% For each token pair $(m,n)$, define the joint heteroscedastic uncertainty
% $\mathbf{\Sigma}_{mn} = \frac{1}{2} \Big[ \sigma^2(\vx_m;\mathcal{P}_\sigma) + \sigma^2(\vx'_n;\mathcal{P}_\sigma) \Big]$,
% where $\sigma^2(\cdot;\mathcal{P}_\sigma)$ is predicted by SigmaNet. 
% This 
The element-wise uncertainty modulates the contribution of each token pair in the alignment, reflecting its reliability.
Using the soft uDTW formulation, the  alignment probability of tokens $m$ and $n$ is given by
$\pi^*_{mn}$ %$ = \sum_{\mPi_i \ni (m,n)} \pi(\mPi_i)$, where $\pi(\mPi_i)$ is the normalized Gibbs distribution over all admissible alignment paths: $\pi(\mPi_i) = \frac{\exp\big(-(d^2_{\mPi_i} - \mu_v)/\gamma\big)}{\sum_{\mPi_j \in \mathcal{A}_{\tau,\tau'}} \exp\big(-(d^2_{\mPi_j} - \mu_v)/\gamma\big)}$.
% Here, $d^2_{\mPi_i} = \langle \mPi_i, \mD \odot \cov^\dag \rangle$ is the precision-weighted path cost, $\cov^\dag = [\sigma_{mn}^{-2}]$ is the element-wise precision matrix, and $\gamma>0$ controls the smoothness of the softmax.  
The resulting matrix $[\pi^*_{mn}]_{(m,n)\in\idx{T}\times\idx{T}}$ defines a structured, uncertainty-aware alignment over token pairs, where low-uncertainty correspondences are automatically assigned higher alignment probability, producing a reliability-weighted, attention-like mapping across the token grid.

The critical insight is that the uncertainty directly modulates the effective attention over token pairs. Low-uncertainty pairs have higher precision $\sigma^{-2}_{mn}$, resulting in larger contributions to $d^2_{\mPi_i}$ and, consequently, larger  probabilities $\pi_{mn}$. Conversely, high-uncertainty pairs contribute less to the alignment and receive lower $\pi_{mn}$. In other words, the learned uncertainty acts as a \textit{reliability-aware gating signal}, producing a \textit{reverse-attention effect}: tokens with low uncertainty dominate the alignment, whereas ambiguous or unreliable tokens are systematically suppressed. Unlike standard attentions, this effect emerges directly from the probabilistic alignment formulation, without any additional attention module.
% Per-token uncertainty maps for the query and support images are obtained by marginalizing $\mathbf{\Sigma}$ over token alignments:
% $\vu_m = \mathbf{\Sigma}_{m\cdot}$, $\vu'_n = \mathbf{\Sigma}_{\cdot n}$, $m,n \in \idx{T}$.
% Reshaping these vectors into the $\sqrt{T} \times \sqrt{T}$ token grid produces patch-wise uncertainty maps $\mU$ and $\mU'$ (normalized to $[0,1]$). Low values indicate reliable, semantically meaningful tokens, while high values indicate ambiguous or noisy tokens. 
% Empirically, these maps show strong correspondence with conventional ViT attention maps, validating that uDTW’s uncertainty naturally induces attention-like focus over visual tokens.
Through this extension, uDTW provides a principled, unified framework for structured alignment over tokenized visual representations. % The probabilistic formulation ensures that low-uncertainty tokens dominate the alignment, yielding both robust structured matching and interpretable, attention-like behavior, directly connecting heteroscedastic uncertainty, alignment, and attention in a mathematically grounded and rigorous manner.

%\vspace{0.1cm}
\noindent\textbf{Semantic structure emergence in visual tokens.}
To better understand the effect of uncertainty-aware alignment on visual representations, we visualize token embeddings before and after the learned uDTW projection. Specifically, we compare raw DINOv3 patch-token embeddings with embeddings transformed by the projection layer trained using the uDTW objective. Token embeddings are visualized by mapping the first three PCA components of the feature space into RGB channels.
Fig. \ref{fig:dinov3-udtw} shows that the original DINOv3 embeddings capture coarse semantic information but often exhibit diffuse activations and weak structural consistency. In contrast, embeddings learned with uDTW become more compact and object-centric, with clearer foreground regions and suppressed background responses.
This effect is particularly pronounced on ultra-fine-grained leaf datasets, where the projected embeddings preserve subtle structural characteristics such as contours, defects, and even vein patterns. These observations suggest that uDTW does not merely reweight pairwise similarities, but instead induces a reliability-aware token geometry that enhances structurally consistent and semantically meaningful correspondences.
Moreover, the projected embeddings exhibit stronger spatial continuity across neighboring regions. Object silhouettes and elongated structures %(\eg, tails, legs, and leaf boundaries) 
become noticeably more coherent after the uDTW projection. This behavior is consistent with the probabilistic alignment formulation, where low-uncertainty correspondences contribute more strongly across admissible alignment paths, while ambiguous or unreliable matches are systematically attenuated.

\begin{figure*}[tbp]
\centering
\subfloat[Time series forecasting.]{\label{fig:1}
\includegraphics[trim=0cm 0cm 0cm 0cm,clip=true,
width=0.23\linewidth]{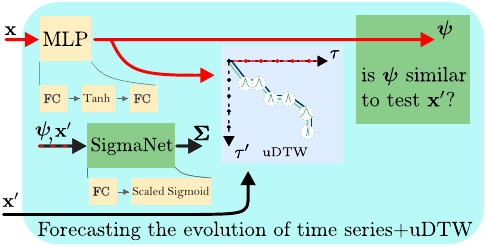}}
\subfloat[Supervised FSAR.]{\label{fig:2}
\includegraphics[trim=0cm 0cm 0cm 0cm,clip=true,
width=0.52\linewidth]{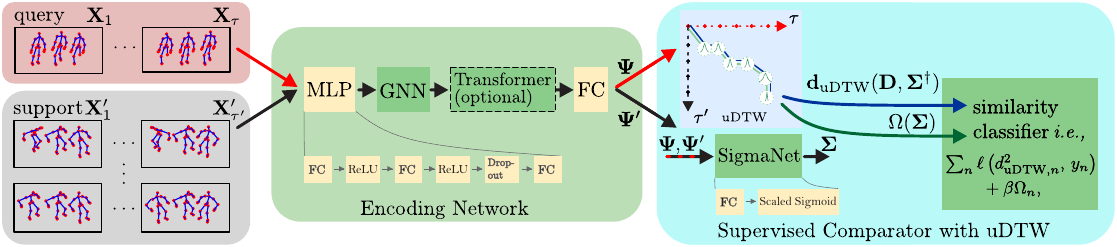}}
% \hfill
\subfloat[Unsupervised FSAR.]{\label{fig:3}
\includegraphics[trim=0cm 0cm 0cm 0cm,clip=true,
width=0.23\linewidth]{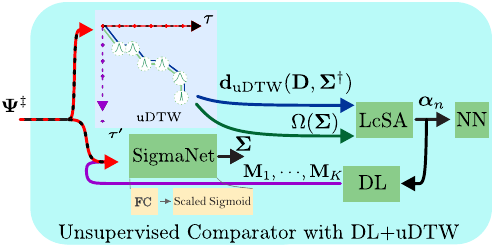}}
% \hfill
% \caption{
% Uncertainty-aware alignment framework across tasks.
% (a) Time series forecasting: predicted futures are aligned with ground truth using uDTW under uncertainty.
% (b) In supervised few-shot skeleton-based action recognition, query and support sequences are partitioned into temporal blocks, encoded into embeddings, and compared using uDTW with learned uncertainty.
% (c) Unsupervised setting: the same encoder is used, while only the comparator is replaced with a label-free objective.
% }
\caption{
Uncertainty-aware alignment across tasks.
(a) Time series forecasting: predicted sequences are aligned with ground truth using uDTW with learned uncertainty.
(b) Supervised few-shot action recognition: support and query sequences are encoded into temporal embeddings and matched via uncertainty-aware alignment.
(c) Unsupervised setting: the same encoder is retained, while alignment is optimized with a label-free objective.
}
\label{fig:pipe_all}
% \vspace{-0.3cm}
\end{figure*}

\section{Learning Framework for Structured Inputs}
\label{sec:general_framework}

Below, we present three representative pipeline formulations where uDTW acts as a core component for computing uncertainty-aware alignment distances over warped paths. %, spanning time series, skeleton sequences, and visual token classification.

\subsection{Time Series Forecasting and Classification}
\label{sec:fsl_ts}

%\vspace{0.1cm}
\noindent
\textbf{Encoder.} Given an observed prefix $\vx \in \mathbb{R}^{t}$ of a time series, %the encoder $f(\cdot; \mathcal{P})$ maps it to a predicted future $\vpsi \in \mathbb{R}^{\tau-t}$ using an MLP.
the encoder $f(\cdot; \mathcal{P})$ (following \cite{marco2017icml}) maps it to a predicted future $\vpsi \in \mathbb{R}^{\tau-t}$ using an MLP, where $\mathcal{P} \equiv [\mathcal{P}_{\text{MLP}}, \mathcal{P}_{\text{SN}}]$ denotes the parameters of the encoder and SigmaNet.
Here, $\vx' \in \mathbb{R}^{\tau-t}$ and $\vpsi \in \mathbb{R}^{\tau-t}$ are vectors representing the ground-truth future and predicted future of the time series, respectively. %, in contrast to $\mPsi$ used in skeleton-based action recognition and image classification tasks, which are matrix-valued token or block embeddings. 
SigmaNet predicts per-timestep positive uncertainties for both % the predicted future 
$\vpsi$ and %the ground-truth future 
$\vx'$. % , enforced via a sigmoid-based transformation as in the image classification setting, ensuring the uDTW regularizer is well-defined.

%\vspace{0.1cm}
\noindent
\textbf{Forecasting.} The model is trained to minimize the uDTW-based forecasting loss across $N$ training series (see Fig. \ref{fig:1}):
\begin{align}
\mathcal{L}_{\text{forecast}} = \sum_{n\in\idx{N}}% ^{N} 
d^2_{\text{uDTW}}(\vpsi_n, \vx'_n) + \beta  \Omega(\vpsi_n, \vx'_n),
\end{align}
where $d^2_{\text{uDTW}}$ measures temporal alignment and $\Omega$ incorporates predictive uncertainty.

%\vspace{0.1cm}
\noindent
\textbf{Classification.} Following \cite{marco2017icml}, we adopt the standard setup for this classical task. For each class $c$, the class prototype $\vmu_c \in \mathbb{R}^{\tau'}$ is estimated as the Fr\'echet mean of all training sequences in that class:
\begin{align}
\vmu_c = \argmin_{\vmu} \sum_{n\in \idx{N_c}} d^2_{\text{uDTW}}(\vx_n, \vmu) + \beta  \Omega(\vx_n, \vmu),
\end{align}
where $\tau'$ is set to the average sequence length across all classes. SigmaNet is trained jointly to predict per-timestep uncertainties for each prototype. 
In addition to the nearest centroid, we also consider the use of nearest neighbors.
Each test sequence $\vx$ is directly compared to all training sequences using uDTW, assigning the label of the closest sequence. % Uncertainty is obtained from the SigmaNet trained on the centroid task.

%\vspace{0.1cm}
\noindent
\textbf{Inference.} Inference follows the task-specific setting. %, including forecasting and classification.
\textit{(i) Forecasting:} For a new prefix $\vx$, the prediction is $\vpsi = f(\vx; \mathcal{P})$.
\textit{(ii) Classification:} Let $\mathcal{C}$ denote the set of classes and $\{\vx_n\}_{n=1}^{N}$ the training set. For nearest centroid,
$\hat{y} = \argmin_{c \in \mathcal{C}} d^2_{\text{uDTW}}(\vx, \vmu_c)\!+\!\beta\Omega(\vx, \vmu_c)$,
and for nearest neighbor,
$\hat{y} = \argmin_{n \in \{1,\dots,N\}} d^2_{\text{uDTW}}(\vx, \vx_n)\!+\!\beta\Omega(\vx, \vx_n)$,
with uncertainties incorporated via SigmaNet.

\subsection{Few-shot Skeleton-based Action Recognition}
\label{sec:fsl_ar}

%\vspace{0.1cm}
\noindent
\textbf{Encoder.} Each skeleton sequence with $J$ joints per frame is partitioned into $\tau$ temporal blocks of length $M$ with stride $S$. Each block $\mX \!\in\! \mathbb{R}^{3 \times J \times M}$ contains 3D joint coordinates. % $(x,y,z)$. 
Each block is encoded by a 3-layer MLP (FC, ReLU, FC, ReLU, Dropout, FC) producing a $d \!\times\!J$ feature map, followed by an S$^2$GC\cite{hao2021iclr} for spatial modeling and a Transformer encoder for temporal interactions. A final FC % layer 
produces the block embedding (see Fig. \ref{fig:2}). Formally, for query and support sequences:
$\mPsi \!=\! f(\mX; \mathcal{P}) \!\in\! \mathbb{R}^{d' \times \tau}$,
$\mPsi' \!=\! f(\mX'; \mathcal{P}) \!\in\! \mathbb{R}^{d' \times \tau'}$,
where the Supervised Comparator $\mathcal{P} \!=\! [\mathcal{P}_{\text{MLP}}, \mathcal{P}_{\text{S$^2$GC}}, \mathcal{P}_{\text{Transf}}, \mathcal{P}_{\text{FC}}, \mathcal{P}_{\text{SN}}]$ includes the parameters of % the 
MLP, S$^2$GC, Transformer, FC, % layer, 
and SigmaNet. SigmaNet predicts per-block scalar uncertainties. % $\sigma \!\in\! \mathbb{R}^{\tau \times 1}$.

%\vspace{0.1cm}
\noindent
\textbf{Supervised.} For $N$-way $Z$-shot episodes with $B$ episodes per batch, each query embedding is randomly sampled from one of the $N$ classes, and support embeddings contain $Z$ sequences per class. Similarity labels are set as $\delta_1 = 0$ for matching classes and $\delta_{n \neq 1} = 1$ for mismatches. The supervised uDTW loss is
\begin{align}
\!\!\!\!\mathcal{L}_{\text{sup}} 
\!=\!\!\! \sum_{b\in\idx{B}} \!\sum_{n\in\idx{N}} \!\sum_{z\in\idx{Z}}\!\!\! 
\left( d^2_{\text{uDTW}}(\mPsi_b, \mPsi'_{b,n,z}) \!-\! \delta_n \right)^2 
\!\!+\! \beta \Omega(\mPsi_b, \mPsi'_{b,n,z}).
\end{align}

%\vspace{0.1cm}
\noindent
\textbf{Unsupervised.} In the unsupervised setting, class labels are ignored. Query and support sequences in each episode are concatenated into $\{\mPsi^\ddag_{b,n}\}_{n=1}^{N\cdot Z + 1}$. Sequences are projected onto a dictionary $\{\mM_k\}_{k=1}^{K}$ using LCSA \cite{6126534}, with coefficients
\begin{align}
& \!\!\!\!%\forall_{k,n},\;
\alpha_{k,b,n}\!\!=\!\!\left\{\begin{array}{l}
\!\!\frac{\exp\big(\!-\frac{1}{\gamma'}d^2_{\text{uDTW}}\big(\mPsi^\ddag_{b,n}, \mM_k\big)\!\big)}{\!\!\!\!\sum\limits_{l\in\mathcal{M}(\mPsi^\ddag_{b,n};K')}\!\!\!\!\!\!\!\!\!\!\!\exp\big(\!-\frac{1}{\gamma'}d^2_{\text{uDTW}}\big(\mPsi^\ddag_{b,n}, \mM_l\big)\!\big)} \quad\!\!\text{if}\;\mM_k\!\in\!\mathcal{M}\!\big(\mPsi^\ddag_{b,n};K'\big),\\
0 \qquad\qquad\qquad\qquad\qquad\qquad\text{otherwise},
\end{array}\right.
    \nonumber\\[-16pt]
    &\label{eq:lcsa} 
\end{align}
where $K' \le K$ is the number of nearest dictionary atoms, $\mathcal{M}(\cdot)$ retrieves the nearest atoms using uDTW, $\tau'$ is set to the mean block count across the training set, and $\gamma' = 0.7$. Dictionary atoms are updated via
\begin{align}
\!\!\!\mM_k \!\gets\! \mM_k \!-\! \lambda_{\text{DL}}\!\!\! \sum_{n=1}^{N\cdot Z + 1}\!\!\! 
\nabla_{\mM_k} d^2_{\text{uDTW}}\Big(\mPsi^\ddag_{b,n}, \sum_{l=1}^{K} \alpha_{l,b,n} \mM_l \Big),
\end{align}
with \texttt{dict\_iter} = 10 and $\lambda_{\text{DL}} = 0.001$. Unsupervised Comparator (see Fig. \ref{fig:3}) is updated correspondingly:
\begin{align}
\!\!\!\!\mathcal{P} \!\gets\! \mathcal{P}\!-\!\lambda_{\text{EN}}\!\!\!\sum_{n=1}^{N\!\cdot\! Z+1}\!\!\!\nabla_{\mathcal{P}}d^2_{\text{uDTW}}\Big(\!\mPsi^\ddag_{b,n}, \mM'\!\Big) \!+\!\beta\Omega\Big(\!\mPsi^\ddag_{b,n}, \mM'\!\Big),
\end{align}
where $\mM'\!=\!\sum_{l=1}^K\alpha_{l,b,n}\mM_{l}$, with $\lambda_{\text{EN}} = 0.001$.

%\vspace{0.1cm}
\noindent
\textbf{Inference.} % During evaluation, 
Query embeddings are matched to support embeddings using uDTW in the supervised setting. In the unsupervised setting, both query and support sequences are encoded into LCSA coefficients and compared using histogram intersection kernel, which measures similarity between their soft assignment distributions. The predicted label is assigned to the support sequence with the highest similarity score.

\begin{figure}[tbp]
\vspace{-0.2cm}
\centering
\subfloat[$\beta$ (where $\lambda\!=\!1$)\label{fig:beta}]{
\includegraphics[trim=1cm 1cm 1cm 2.5cm,clip=true,
width=0.485\linewidth]{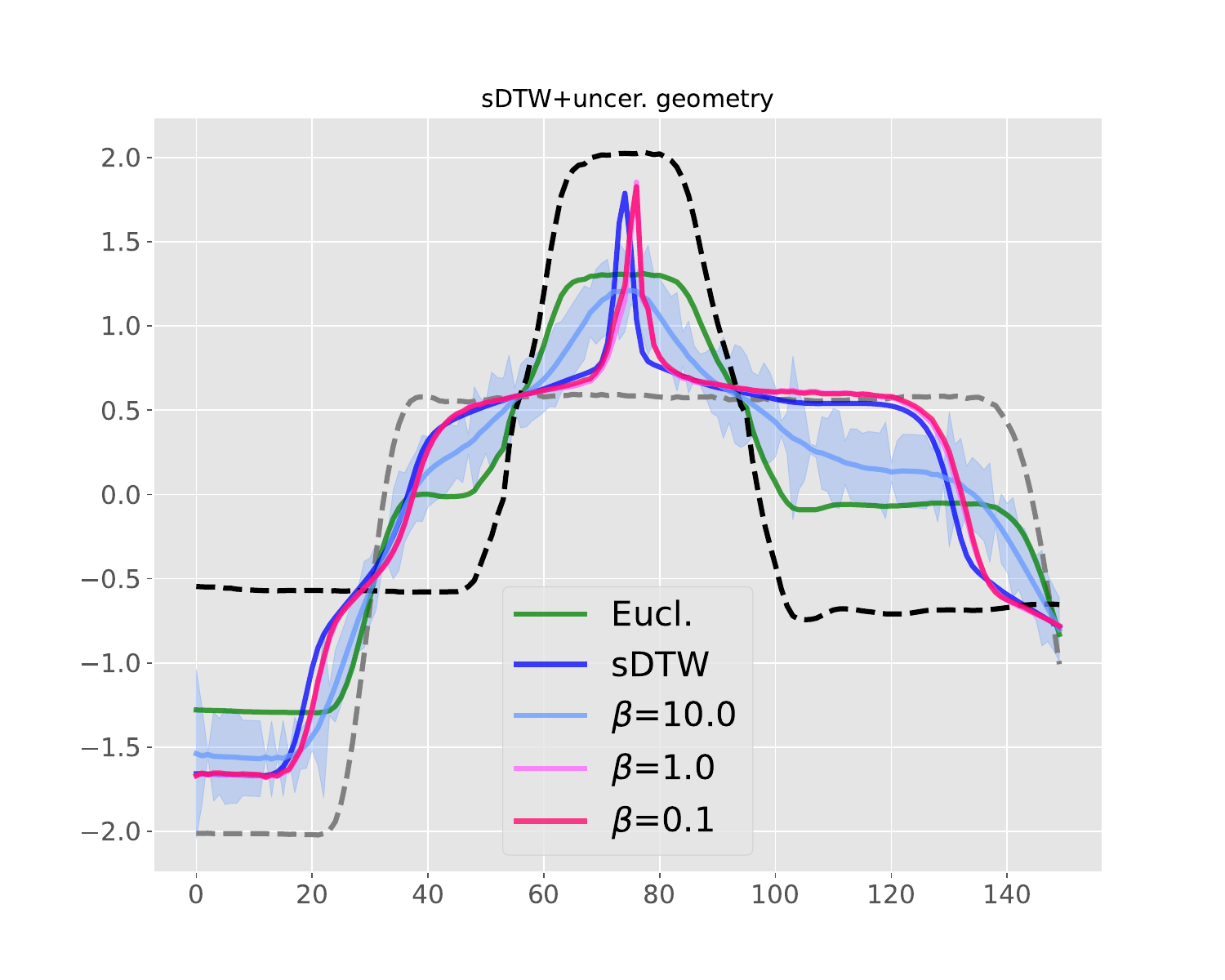}}
\hfill
\subfloat[$\lambda$ (where $\beta\!=\!10$)\label{fig:lambda}]{
\includegraphics[trim=1cm 1cm 1cm 2.5cm,clip=true,
width=0.485\linewidth]{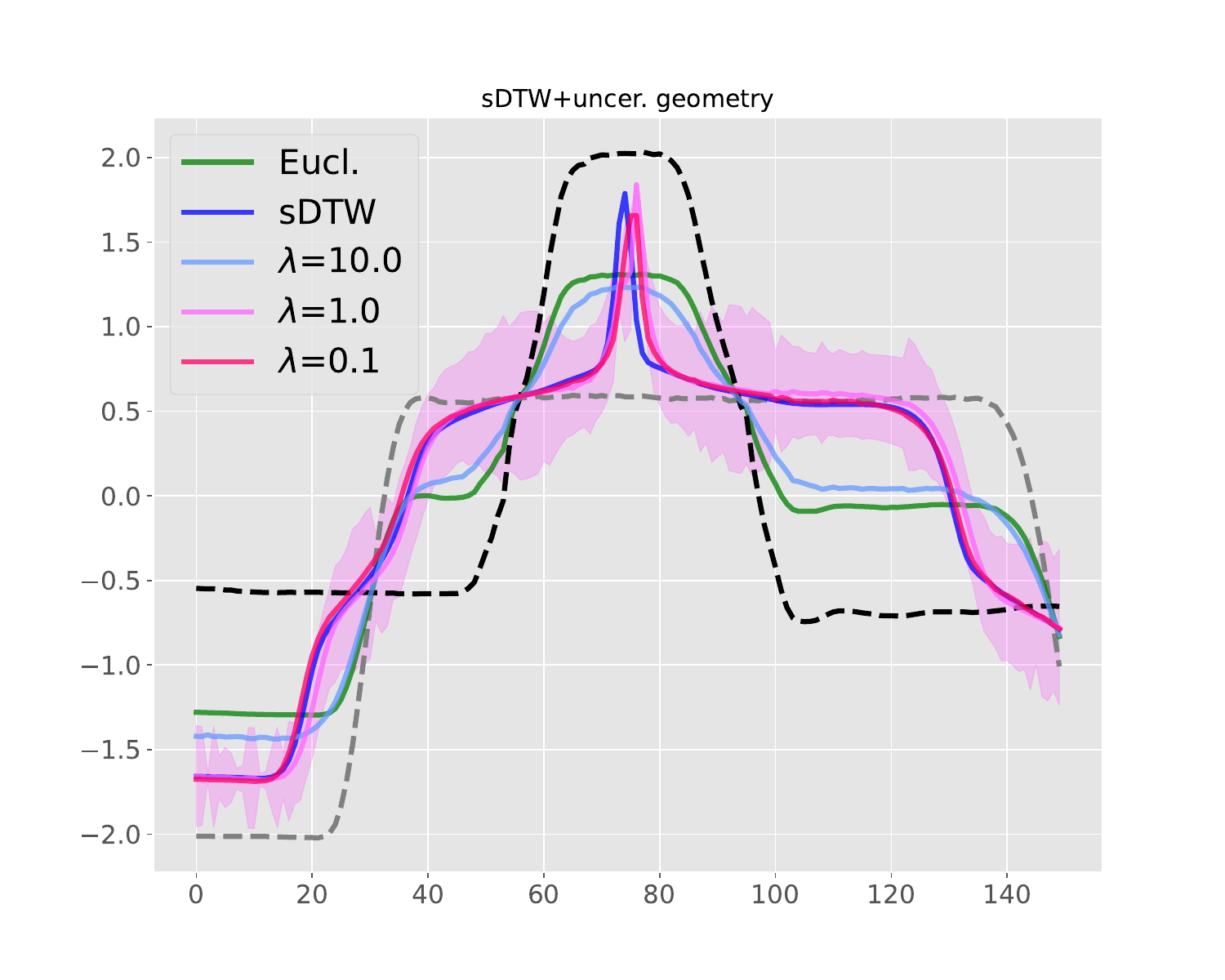}}
% \caption{Interpolation between two time series (grey and black dashed lines) 
% on the Gun Point dataset. 
% We compute the barycenter by solving 
% $\argmin\limits_{\vmu, \vsigma_\mu}\sum_{n=1}^{2}
% d^2_{\text{uDTW}}\big( \mD, \cov^\dag\big)+\beta\Omega\big(\cov\big)+\lambda\Omega'\big(\cov\big)$ where  $\mD\!=\!(\vx_n\mathbf{1}^\top\!-\!\mathbf{1}\vmu^\top)^2$ and  $\cov\!=\!\mathbf{1}\mathbf{1}^\top\!\!+\!\mathbf{1}\vsigma_\mu^\top$
% where $\vx_n$ is the given $n$-th time series. 
% $\beta\!\geq\!0$ controls the penalty for high matching uncertainty,
% $\Omega'$ is defined as in Eq. \eqref{eq:soft_uDTW_compact} but element-wise $\log\!\cov$ is replaced by element-wise $(\cov\!-\!1)^2$ so that
%  $\lambda\!\geq\!0$ favours uncertainty to remain close to one. 
%  $\beta$ and $\lambda$ control the uncertainty estimation and yield  different barycenters than  the Euclidean (green color) and sDTW (blue color) distances. As $\Omega$ and $\Omega'$ act similar, we only use $\Omega$ in our experiments.
% }
\caption{
Interpolation between two time series (grey and black dashed) on the Gun Point dataset. 
The barycenter is obtained by minimizing an uncertainty-aware uDTW objective, where $\beta$ penalizes high matching uncertainty and $\lambda$ regularizes it toward unit variance. 
Compared to Euclidean (green) and sDTW (blue) barycenters, uDTW produces distinct interpolations by adaptively weighting alignments according to uncertainty. 
}
\label{fig:interp}
% \vspace{-0.2cm}
\end{figure}

\subsection{Few-shot Image Classification}
\label{sec:fsl_image}

%\vspace{0.1cm}
\noindent
\textbf{Encoder.}
We adopt a pretrained vision backbone (\eg, DINOv3 with a ViT encoder), as few-shot learning benefits from transferable representations learned from large-scale datasets. % such as ImageNet. 
Given an input image, we extract the last-layer patch tokens and discard prefix tokens (\eg, class and register tokens). 
The resulting token sequence is projected into a $d'$-dimensional embedding space using a single bias-free linear layer without activation.
To model reliability, we use a SigmaNet that predicts a positive scalar uncertainty for each token. 
The set of learnable parameters is 
$\mathcal{P} \equiv [\mathcal{P}_{\text{proj}}, \mathcal{P}_{\text{SN}}]$, 
where $\mathcal{P}_{\text{proj}}$ denotes the projection layer and $\mathcal{P}_{\text{SN}}$ the uncertainty predictor.

%\vspace{0.1cm}
\noindent
\textbf{Supervised.}
% We follow the standard $N$-way $Z$-shot episodic setting. Each episode contains one query and $N \times Z$ support samples drawn from $N$ classes. A mini-batch consists of $B$ episodes.
Let $\{\mPhi_b\}_{b\in\idx{B}}$ denote the query feature maps and $\{\mPhi'_{b,n,z}\}_{b\in\idx{B},\,n\in\idx{N},\,z\in\idx{Z}}$ the corresponding support feature maps. For each episode $b$, the query shares the same label as the supports from class $n=1$, \ie,
$y(\mPhi_b) = y(\mPhi'_{b,1,z})$, $\forall z \in \idx{Z}$,
while supports from $n \neq 1$ belong to different classes. Accordingly, we define similarity labels $\delta_1 = 0$ for matching pairs and $\delta_{n\neq1} = 1$ otherwise.
The supervised objective minimizes a regression loss over pairwise distances:
\begin{align}
% \!\!
\!\!\mathcal{L}_{\text{sup}} 
\!=\!\!\! \sum_{b\in\idx{B}} \!\sum_{n\in\idx{N}} \!\sum_{z\in\idx{Z}}\!\!\! 
\left( d^2_{\text{uDTW}}(\mPhi_b, \mPhi'_{b,n,z}) \!-\! \delta_n \right)^2 
\!\!+\! \beta \Omega(\mPhi_b, \mPhi'_{b,n,z}).
\end{align}

%\vspace{0.1cm}
\noindent
\textbf{Unsupervised.} The Unsupervised Comparator is trained in a label-free (unsupervised) setting. 
Given a mini-batch of $B$ images, we generate two stochastic augmented views per image, yielding paired token embeddings $\{(\mPhi_i, \mPhi_i')\}_{i=1}^{B}$. 
The model is optimized using an uncertainty-aware objective:
\begin{equation}
% \!\!\!\!\!\!\!\!\!\!\!\!\!\!\!\!\
\!\!\mathcal{L}_{\text{unsup}}
\!\!= \!\!\frac{1}{B}\!\sum_{i\in\idx{B}}\!\left[\!- 
\log\!\frac{\exp\!\left(-d^2_{\text{uDTW}}(\mPhi_i, \mPhi_i')/\tau\right)}
{\sum_{j=1}^{B} \exp\!\left(-d^2_{\text{uDTW}}(\mPhi_i, \mPhi_j')/\tau\right)}
\!+\! {\beta} \Omega(\mPhi_i, \mPhi_i')\!\right],
\end{equation}
where $\tau$ is the temperature. Positive pairs, corresponding to two views of the same image, are pulled together in the embedding space, while the off-diagonal terms in the denominator serve as in-batch negatives and are pushed apart. 
This InfoNCE-style loss provides a principled framework for self-supervised learning of token embeddings with structured alignment via uDTW. 
After unsupervised pretraining, the learned projection and SigmaNet are frozen and reused in standard few-shot prototype evaluation.

%\vspace{0.1cm}
\noindent
\textbf{Inference.}  
During evaluation, class prototypes % $\mPhi_c$ 
are computed as the mean of the token embeddings in the support set for each class. 
In the supervised setting, query features can be matched either to individual support features or to these class prototypes using uDTW. 
In the unsupervised setting, both support and query samples are encoded using the frozen projector. 
Classification is then performed by computing the uDTW distance between each query % $\mPhi$ 
and the class prototypes, % $d^2_{\text{uDTW}}(\mPhi, \mPhi_c)$, 
and assigning the query to the class with the minimum distance. 
% The uncertainty regularizer is not applied during evaluation; it is used only during training to improve embedding robustness, ensuring that the few-shot decision rule aligns with standard prototype-based classification.

% \textbf{Dataset granularity.}  
% This pipeline applies uniformly across coarse/generic, fine, and ultra-fine-grained datasets. Support embeddings form class prototypes at the appropriate granularity, with $N$-way $K$-shot episodes adjusted accordingly. No changes to the encoder or training procedure are required.

\begin{figure}[tbp]
\centering
\subfloat[CBF\label{fig:cbf}]{
\includegraphics[trim=4cm 1.2cm 3.5cm 0.6cm,clip=true,
width=\linewidth]{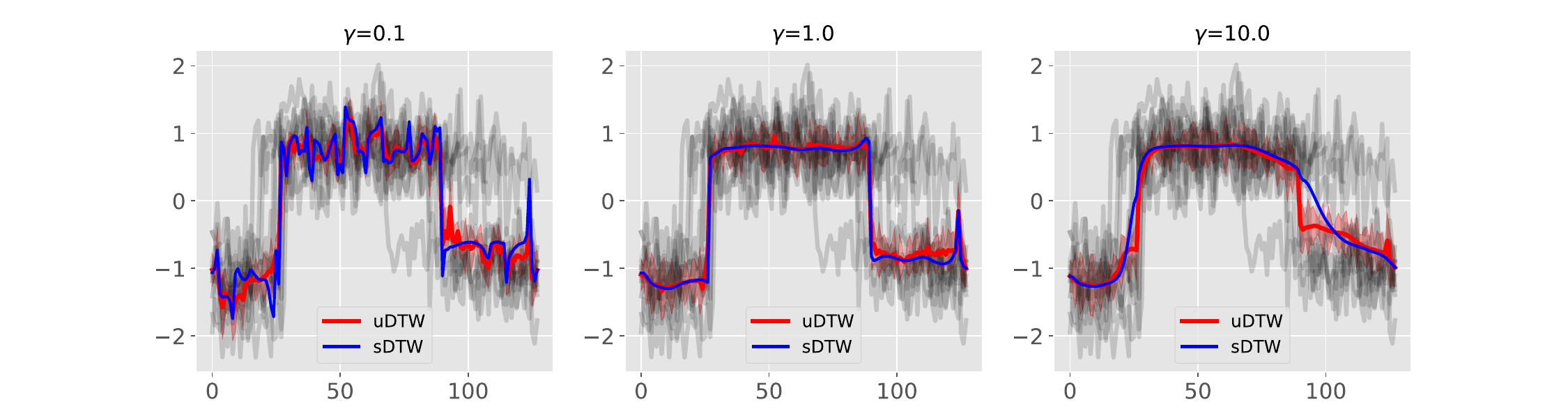}}\\% \hfill
\vspace{-0.2cm}
\subfloat[Synthetic Control\label{fig:syn}]{
\includegraphics[trim=4cm 1.2cm 3.5cm 0.6cm,clip=true,
width=\linewidth]{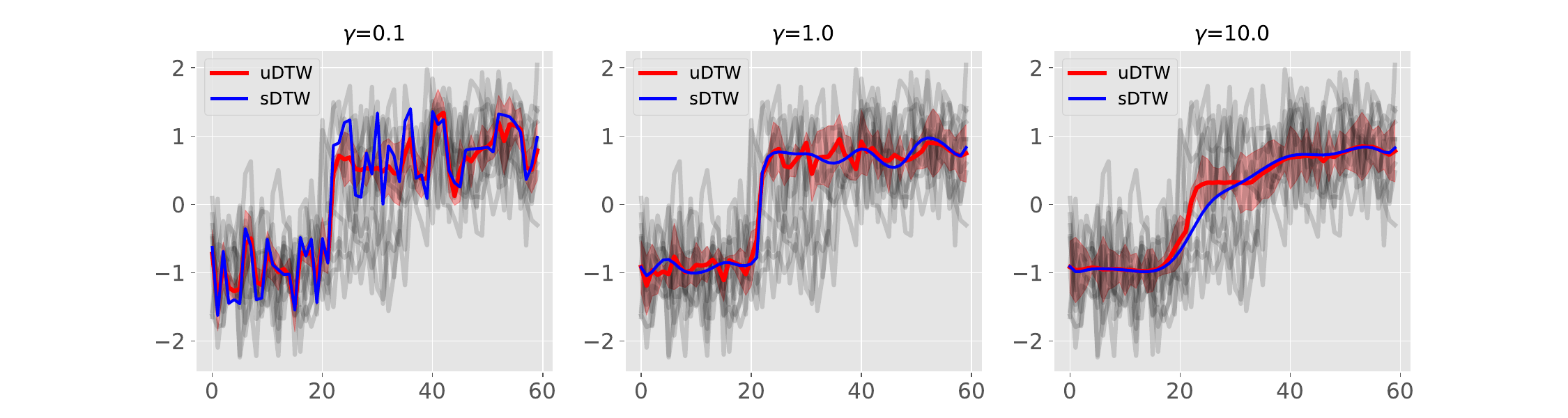}}
% \caption{Comparison of barycenter based on sDTW or uDTW on CBF  
% and Synthetic Control. 
% We visualize uncertainty around the barycenters in red color for uDTW. Our uDTW generates reasonable barycenters even when higher $\gamma$ values are used, \eg, $\gamma\!=\!10.0$. Higher $\gamma$ value leads to smooth barycenter but introducing higher uncertainty.
% }
\caption{Comparison of barycenters computed with sDTW and uDTW on CBF and Synthetic Control. Uncertainty for uDTW is shown in red. Even with high $\gamma$ (\eg, $\gamma\!=\!10.0$), uDTW produces reasonable barycenters; higher $\gamma$ smooths the barycenter but increases uncertainty.}
\label{fig:barycenter}
\vspace{-0.2cm}
\end{figure}

\section{Experiment}

%Evaluations and analyses are presented below.

% We evaluate uDTW across multiple tasks, including time series analysis, few-shot action recognition, and few-shot image classification spanning generic, fine-grained, and ultra-fine-grained categories. 
% We summarize datasets, evaluation protocols, and baselines below.

\subsection{Datasets, Setups, and Evaluation Protocols}

%\vspace{0.1cm}
\noindent
\textbf{Time series analysis.} We use the {\em UCR archive}~\cite{UCRArchive}, a standard benchmark %for time series classification and forecasting, 
which contains sequences from diverse domains such as astronomy, geology, and medical imaging, with variable lengths and complexity. For classification, we follow standard train/validation/test splits and evaluate nearest-neighbor and nearest-centroid classifiers. Forecasting is evaluated by using the first portion (60\%) of each sequence as input and the remainder (40\%) as output, with metrics including MSE and alignment-based distances (\eg, sDTW). % (DTW, sDTW div., and uDTW). 
Barycenter computation is performed following sDTW~\cite{marco2017icml}. % Baselines include Euclidean distance, DTW~\cite{marco2011icml}, and sDTW~\cite{marco2017icml}, and sDTW div.~\cite{pmlr-v130-blondel21a}.

%\vspace{0.1cm}
\noindent
\textbf{Few-shot action recognition (sequences).} Skeleton-based action recognition is evaluated on (i) {\em NTU-60}~\cite{Shahroudy_2016_NTURGBD}, which contains 56,880 RGB+D sequences over 60 action classes with high intra-class variability, (ii) {\em NTU-120}~\cite{Liu_2019_NTURGBD120}, an extension to 120 classes and 114,480 sequences captured from 106 subjects across 155 camera viewpoints, and (iii) {\em Kinetics-skeleton}~\cite{kay2017kinetics}, a large-scale human action dataset. % For Kinetics, 2D joint coordinates are extracted using OpenPose \cite{cao2019openpose}, and 3D coordinates are estimated via the network of Martinez et al.~\cite{martinez_2d23d} pre-trained on Human3.6M~\cite{Catalin2014Human3}. 
Evaluation follows few-shot protocols \cite{Liu_2019_NTURGBD120, wang2024meet}. %, typically 1-/5-shot, with training and test splits ensuring no subject overlap. 
Baselines include Matching Networks~\cite{miniimagenet_nips}, Prototypical Networks~\cite{snell2017prototypical}, and TAP~\cite{su2022temporal}, chosen for their strong performance in few-shot action recognition.

%\vspace{0.1cm}
\noindent
\textbf{Few-shot image classification.} For \textit{generic} few-shot classification, we evaluate on CIFAR-FS~\cite{cifarfs_dataset}, miniImageNet~\cite{miniimagenet_nips}, and tieredImageNet~\cite{tieredimagenet_iclr}, which contain diverse object categories with moderate inter-class variation. \textit{Fine-grained} datasets include CUB-200-2011~\cite{cub2002011}, Stanford Dogs~\cite{stanforddogs}, Stanford Cars~\cite{stanfordcars}, Aircraft~\cite{aircraft}, and NABirds~\cite{nabirds}, which require discriminative localized features to distinguish subtle differences within a category. \textit{Ultra-fine-grained} (UFGVC) datasets \cite{ufgvc}, such as Cotton80 and SoyLocal, feature minimal visual variation and very few samples per class. 
Standard evaluation uses 5-way 1-/5-shot episodes for generic and fine-grained datasets, and 5-way 1-/2-shot for UFGVC datasets due to the limited number of samples per class. \footnote{We define a new few-shot protocol on Cotton80 and SoyLocal under 5-way 1-shot and 5-way 2-shot settings, with classes evenly split into base and novel categories for both supervised and unsupervised evaluation.% The same evaluation protocol is applied to both supervised and unsupervised settings. %, with no validation set.
} In the unsupervised setting, evaluations include Omniglot~\cite{lake2015human} and miniImageNet for generic few-shot tasks, as well as cross-domain tests on CUB-200-2011, Stanford Cars, Cotton80 and SoyLocal, where embeddings are learned on % a source dataset 
miniImageNet in an unsupervised manner, frozen, and then applied to target-domain few-shot episodes. Baselines comprise supervised methods such as LA-PID~\cite{lapid_tpami}, CPEA~\cite{ceta_iccv}, ASCM~\cite{ascm_pr}, as well as unsupervised approaches including CACTUs\cite{cactu_iclr19}, UMTRA\cite{umtra_nips19}, LASIUM\cite{lasium_iclr21}, PsCo\cite{psco_iclr23}, Meta-GMVAE\cite{metagmvae_iclr2021}, and Meta-SVEBM\cite{svebm_nipsw21}. We adopt DINOv3 (ViT-S/16), pretrained on LVD-1689M~\cite{dinov3}, as a frozen feature backbone. On top, we train a projection layer (single bias-free FC, 384$\rightarrow$256) for each objective, and an additional SigmaNet for uDTW (a single FC, 256$\rightarrow$256, followed by feature-wise average pooling and a scaled sigmoid, as in~\cite{wang2022uncertainty}). 
We evaluate three objectives: squared Euclidean distance, sDTW, and uDTW. Euclidean distance operates on pooled class prototypes, whereas sDTW and uDTW perform token-level alignment over patch embeddings, with uDTW further using per-token uncertainty.% These baselines represent state-of-the-art supervised and unsupervised few-shot methods, providing a comprehensive benchmark for evaluating uDTW.

\begin{table}[tbp]
% \caption{Classification accuracy (mean$\pm$std) on UCR archive by the nearest neighbor and the nearest centroid classifiers.
% In the column we indicate which distance was used for computing the class prototypes. $K$ is the number of nearest neighbors in this context.}
\caption{Classification accuracy (mean$\pm$std) on the UCR archive. Columns indicate the distance used for class prototypes; $K$ denotes the number of nearest neighbors.}
\vspace{-0.5cm}
\label{tab:ucrclassification}
\begin{center}
\resizebox{\linewidth}{!}{\begin{tabular}{lcccc}
\toprule
 & \multicolumn{3}{c}{\bf Nearest neighbor} &\multirow{2}{*}{\bf Nearest centroid}\\
 \cmidrule{2-4}
 & $K = 1$ & $K = 3$ & $K = 5$ & \\
\midrule
Euclidean & 71.2$\pm$17.5& 72.3$\pm$18.1& 73.0$\pm$16.7&61.3$\pm$20.1\\
DTW~\cite{marco2011icml} & 74.2$\pm$16.6&75.0$\pm$17.0&75.4$\pm$15.8&65.9$\pm$18.8 \\
sDTW~\cite{marco2017icml} & 76.2$\pm$16.6&77.2$\pm$15.9& 78.0$\pm$16.5&70.5$\pm$17.6\\ 
sDTW div.~\cite{pmlr-v130-blondel21a}& 78.6$\pm$16.2& 79.5$\pm$16.7&80.1$\pm$16.5\!\!& 70.9$\pm$17.8\\ 
\hdashline

uDTW & 80.0$\pm$15.0& 81.2$\pm$17.8&83.3$\pm$16.2& 72.2$\pm$16.0\\
\bottomrule
\end{tabular}}
\end{center}
\vspace{-0.6cm}
\end{table}

\begin{figure}[tbp]
\vspace{-0.2cm}
\centering
\subfloat[ECG200\label{fig:ecg200_pred}]{
\includegraphics[trim=0.6cm 9.2cm 12.6cm 0cm,clip=true,
width=0.485\linewidth]{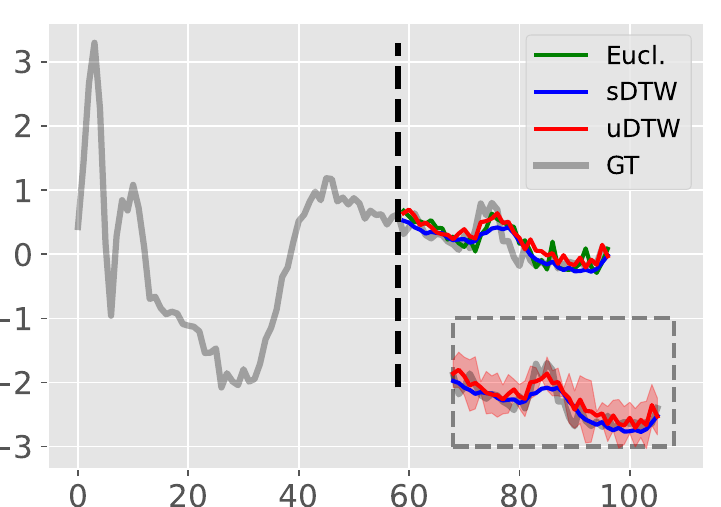}}
\hfill
\subfloat[ECG5000\label{fig:ecg5000_pred}]{
\includegraphics[trim=12.6cm 0.5cm 0.6cm 9cm,clip=true,
width=0.485\linewidth]{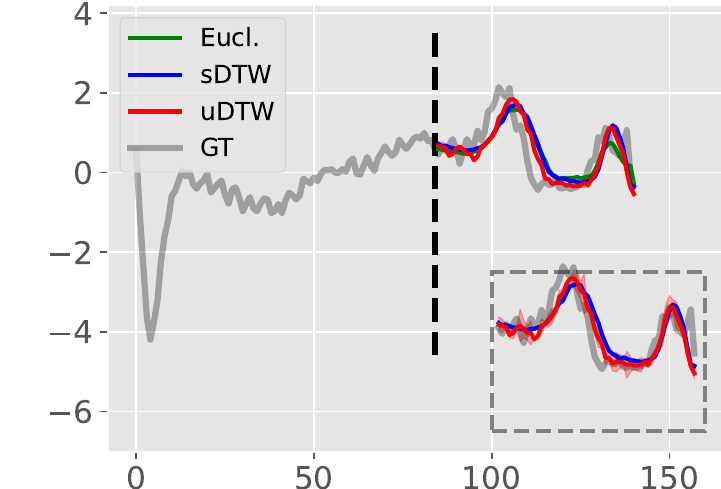}}

% \caption{Given the first part of a time series, we train 3 multi-layer perception (MLP) to predict the remaining part, we use the Euclidean, sDTW or uDTW distance per MLP. We use ECG200 and ECG5000 in UCR archive,  and display  the prediction obtained for the given test sample with either of these 3 distances and the ground truth (GT). Oftentimes, we observe that uDTW  helps predict the sudden changes well.}
\caption{Prediction of time series continuation using MLPs with Euclidean, sDTW, or uDTW distances. Given the first part of a series, three MLPs predict the remaining values on ECG200 and ECG5000 (UCR archive). Red shading indicates uDTW uncertainty. Predictions and ground truth (GT) for a sample are shown. uDTW often captures sudden changes more accurately.}
\label{fig:pred}
\vspace{-0.2cm}
\end{figure}

\begin{table}[tbp]
% \caption{Time series forecasting results evaluated with MSE, DTW, sDTW div. and uDTW metrics on ECG5000, averaged over 100 runs (mean$\pm$std). 
% Best method(s) are highlighted in bold using Student's $t$-test.
% Column-wise distances indicate the distance used during training. Row-wise distances indicate the distance used to compare prediction with the groundtruth at the test time (lower values are better). 
% }
\caption{Time series forecasting on ECG5000 (100 runs, mean$\pm$std), with column-wise distances used for training and row-wise distances for test-time comparison with ground truth (lower is better). Best method(s) in bold (Student's $t$-test).}
\vspace{-0.5cm}
\label{tab:pred}
\begin{center}
\begin{tabular}{lcccc}
\toprule
&  MSE & DTW & sDTW div. & uDTW\\
\hline
Euclidean & {\bf 32.1$\pm$1.62}&20.0$\pm$0.18& 15.3$\pm$0.16 & 14.4$\pm$0.18\\
sDTW~\cite{marco2017icml} &38.6$\pm$6.30&{\bf 17.2$\pm$0.80}&22.6$\pm$3.59& 32.1$\pm$2.25\\
sDTW div.~\cite{pmlr-v130-blondel21a} &24.6$\pm$1.37&38.9$\pm$5.33&{\bf 20.0$\pm$2.44}& 15.4$\pm$1.62\\

uDTW &23.0$\pm$1.22&{\bf 16.7$\pm$0.08}& {\bf 16.8$\pm$1.62}&{\bf 8.27$\pm$0.79} \\
\bottomrule
\end{tabular}
\end{center}
\vspace{-0.5cm}
\end{table}

\subsection{Analysis and Prediction of Time Series}

\subsubsection{Fr\'echet Mean of Time Series} We visualize the Fr\'echet mean under Euclidean distance, sDTW, and our uDTW. % We visually inspect the Fr\'echet mean computed with the Euclidean distance, sDTW, and our uDTW.

%\vspace{0.1cm}
\noindent
\textbf{Setup.} Following \cite{marco2017icml}, for each dataset in UCR, we randomly select a class and pick 10 time series to compute its barycenter. Our uDTW barycenter objective is minimized using L-BFGS \cite{lbfgs} with a maximum of 100 iterations.

%\vspace{0.1cm}
\noindent
\textbf{Qualitative results.} First, we average two time series (Fig.~\ref{fig:interp}). Averaging under uDTW produces substantially different results compared to Euclidean and sDTW. Fig.~\ref{fig:barycenter} shows barycenters from sDTW and uDTW. Our uDTW yields more reasonable barycenters, even for large $\gamma$ (\eg, $\gamma=10$, right column), where change points in the red curve appear sharper. 
With low smoothing ($\gamma=0.1$), both uDTW and sDTW can get stuck in poor local minima, but uDTW produces fewer sharp peaks due to the uncertainty measure. Higher $\gamma$ values smooth the barycenter but increase uncertainty (see comparison of $\gamma=0.1$ \vs $\gamma=10.0$). At $\gamma=1$, barycenters from sDTW and uDTW align well with the original time series.

\subsubsection{Classification of Time Series}   

We evaluate uDTW on UCR time series classification, using a $K$-nearest neighbor classifier with softmax for prediction.

%\vspace{0.1cm}
\noindent
\textbf{Setup.} Each dataset is split 50/25/25 (train/val/test). % We evaluate $K \in \{1,2,3\}$.% Each dataset is split into 50\% training, 25\% validation, and 25\% testing. 
The classifier is evaluated for $K \in \{1,2,3\}$.

%\vspace{0.1cm}
\noindent
\textbf{Quantitative results.} Table~\ref{tab:ucrclassification} compares uDTW with Euclidean, DTW, sDTW, and sDTW div. Using uDTW to compute barycenters improves the nearest centroid classifier by approximately 2\% over sDTW div. For the $K$-nearest neighbor classifier, uDTW increases accuracy by 1.4\%, 1.7\%, and 3.2\% for $K=1,2,3$, respectively.

\subsubsection{Forecasting the Evolution of Time Series}

We evaluate uDTW for forecasting future values of time series, aiming to capture both smooth trends and abrupt changes.

%\vspace{0.1cm}
\noindent
\textbf{Qualitative results.} Fig.~\ref{fig:pred} visualizes the predictions. sDTW and uDTW predictions can differ, but uDTW better matches the ground truth and detects sharp changes more accurately. % While sDTW and uDTW predictions sometimes agree, they can differ noticeably, with uDTW more confidently predicting sharp changes.

%\vspace{0.1cm}
\noindent
\textbf{Quantitative results.} We validate uDTW on ECG5000 from the UCR archive, which contains 5000 electrocardiograms (500 for training and 4500 for testing), each of length 140. 
To evaluate predictions, we use two metrics: (i) MSE for per-time-step errors, and (ii) DTW, sDTW div., and uDTW to compare the overall shape of the time series. Shape-based metrics are important because varying series lengths can bias MSE, ignoring trends in the signal. We apply a Student's $t$-test (significance level 0.05) over 100 runs to highlight the best performance. Table~\ref{tab:pred} shows that uDTW achieves near-best results on both MSE and shape metrics. % (lower scores indicate better performance).

\begin{table}[tbp]
\centering
\caption{Evaluations on NTU-60 (left) and NTU-120 (right).}
\vspace{-0.2cm}
\label{tab:ntu_combined}
\resizebox{\linewidth}{!}{
\begin{tabular}{lcccccccccc}
\toprule
& \multicolumn{5}{c}{NTU-60} & \multicolumn{5}{c}{NTU-120} \\
\cmidrule(lr){2-6} \cmidrule(lr){7-11}
\#classes 
& 10 & 20 & 30 & 40 & 50 
& 20 & 40 & 60 & 80 & 100 \\
\midrule

& \multicolumn{10}{c}{Supervised} \\
\midrule
MatchNets~\cite{miniimagenet_nips} 
& 46.1 & 48.6 & 53.3 & 56.3 & 58.8 
& 20.5 & 23.4 & 25.1 & 28.7 & 30.0 \\

ProtoNet~\cite{snell2017prototypical} 
& 47.2 & 51.1 & 54.3 & 58.9 & 63.0 
& 21.7 & 24.0 & 25.9 & 29.2 & 32.1 \\

TAP~\cite{su2022temporal}
& 54.2 & 57.3 & 61.7 & 64.7 & 68.3 
& 31.2 & 37.7 & 40.9 & 44.5 & 47.3 \\

\hdashline
Euclidean 
& 38.5 & 42.2 & 45.1 & 48.3 & 50.9
& 18.7 & 21.3 & 24.9 & 27.5 & 30.0 \\

sDTW~\cite{marco2017icml} 
& 53.7 & 56.2 & 60.0 & 63.9 & 67.8
& 30.3 & 37.2 & 39.7 & 44.0 & 46.8 \\

sDTW div.~\cite{pmlr-v130-blondel21a} 
& 54.0 & 57.3 & 62.1 & 65.7 & 69.0
& 30.8 & 38.1 & 40.0 & 44.7 & 47.3 \\

uDTW 
& 56.9 & 61.2 & 64.8 & 68.3 & 72.4
& 32.2 & 39.0 & 41.2 & 45.3 & 49.0 \\

\midrule
& \multicolumn{10}{c}{Unsupervised} \\
\midrule
Euclidean 
& 20.9 & 23.7 & 26.3 & 30.0 & 33.1
& 13.5 & 16.3 & 20.0 & 24.9 & 26.2 \\

sDTW~\cite{marco2017icml} 
& 35.6 & 45.2 & 53.3 & 56.7 & 61.7
& 20.1 & 25.3 & 32.0 & 36.9 & 40.9 \\

sDTW div.~\cite{pmlr-v130-blondel21a} 
& 36.0 & 46.1 & 54.0 & 57.2 & 62.0
& 20.8 & 26.0 & 33.2 & 37.5 & 42.3 \\

uDTW 
& 37.0 & 48.3 & 55.3 & 58.0 & 63.3
& 22.7 & 28.3 & 35.9 & 39.4 & 44.0 \\

\bottomrule
\end{tabular}}
\vspace{-0.3cm}
\end{table}

\begin{table}[tbp]
\caption{Evaluations on 2D and 3D Kinetics-skeleton.}
\vspace{-0.5cm}
\label{tab:kinetics}
\begin{center}
\begin{tabular}{lcccc}
\toprule
& \multicolumn{2}{c}{Supervised} & \multicolumn{2}{c}{Unsupervised}\\
\cmidrule(lr){2-3}\cmidrule(lr){4-5}
& 2D & 3D & 2D & 3D\\
\midrule
Euclidean & 21.2 & 23.1 & 12.7 & 13.3\\
TAP~\cite{su2022temporal}
& 32.9 & 36.0 & - & -\\
sDTW~\cite{marco2017icml}  & 34.7 & 39.6 & 23.3& 28.3\\
sDTW div.~\cite{pmlr-v130-blondel21a}  & 35.0 & 40.1 & 24.0& 28.9\\

uDTW & 35.5 & 42.0 & 25.9 & 32.7\\
\bottomrule
\end{tabular}
\end{center}
\vspace{-0.7cm}
\end{table}

\begin{table}[tbp]
\centering
\caption{Few-shot accuracy (\%) on generic image classification benchmarks. Best in bold, second-best underlined.}
\vspace{-0.2cm}
\label{tab:generic}
\resizebox{\linewidth}{!}{\begin{tabular}{llcc}
\toprule
Dataset & Method & 1-shot & 5-shot \\
\midrule
\multirow{8}{*}{CIFAR-FS}
& CPEA \cite{ceta_iccv}& ${77.82 \pm 0.66}$ & ${88.98 \pm 0.45}$ \\
& ASCM \cite{ascm_pr}& $77.25 \pm 0.75$ & ${88.83 \pm 0.58}$ \\
% & NegCosIC \cite{negcosic}& $73.98 \pm 0.65$ & $87.33 \pm 0.48$ \\
& L2TT \cite{l2tt_siam}& $73.60 \pm 0.10$ & $85.80 \pm 0.10$ \\
& LA-PID \cite{lapid_tpami}& $73.20 \pm 0.40$ & $85.80 \pm 0.30$ \\
& SSL-ProtoNet \cite{sslprotonet_esa}& $60.41 \pm 0.52$ & $76.52 \pm 0.38$ \\
\cdashline{2-4}
& Euclidean & $74.01 \pm 0.86$ & $85.45 \pm 0.57$ \\
& sDTW & \underline{$79.96 \pm 0.83$} & \underline{$90.24 \pm 0.50$} \\
& uDTW & $\mathbf{80.32 \pm 0.81}$ & $\mathbf{91.15 \pm 0.50}$ \\
\midrule
\multirow{14}{*}{miniImageNet}
& SemFew-Trans \cite{semfew_cvpr}& $78.94 \pm 0.66$ & $86.49 \pm 0.50$ \\
% & SemFew \cite{semfew_cvpr}& $77.63 \pm 0.63$ & $83.04 \pm 0.48$ \\
& LA-PID \cite{lapid_tpami}& $72.20 \pm 0.50$ & $87.30 \pm 0.50$ \\
& CPEA \cite{ceta_iccv}& $71.97 \pm 0.65$ & $87.06 \pm 0.38$ \\
% & EPNet \cite{epnet_jmlr22}& $70.74 \pm 0.85$ & $84.34 \pm 0.53$ \\
& CORE \cite{core_tpami}& $70.50 \pm 0.44$ & $86.28 \pm 0.29$ \\
& ANROT-HELANet \cite{anrothelanet}& $69.40 \pm 0.30$ & ${88.10 \pm 0.40}$ \\
& ASCM \cite{ascm_pr}& $69.35 \pm 0.65$ & $84.87 \pm 0.58$ \\
% & COSOC \cite{cosoc_nips}& $69.28 \pm 0.49$ & $85.16 \pm 0.42$ \\
& FA-adapter \cite{faadapter_pr}& $68.36 \pm 0.45$ & $83.44 \pm 0.28$ \\
& TNPNet \cite{tnpnet_nc}& $68.22 \pm 0.87$ & $82.50 \pm 0.82$ \\
& HELA-VFA \cite{helavfa_wacv}& $68.20 \pm 0.30$ & $86.70 \pm 0.70$ \\
& MetaDiff \cite{metadiff_aaai}& $64.99 \pm 0.77$ & $81.21 \pm 0.56$ \\
& SSL-ProtoNet \cite{sslprotonet_esa}& $52.58 \pm 0.45$ & $70.87 \pm 0.36$ \\
\cdashline{2-4}
& Euclidean & $81.66 \pm 0.81$ & $92.56 \pm 0.42$ \\
& sDTW & $\underline{85.82 \pm 0.69}$ & $\underline{95.00 \pm 0.28}$ \\
& uDTW & $\mathbf{88.05 \pm 0.63}$ & $\mathbf{95.59 \pm 0.28}$ \\

\midrule
\multirow{11}{*}{tieredImageNet}
& SSL-ProtoNet \cite{sslprotonet_esa}& $55.14 \pm 0.49$ & $74.23 \pm 0.40$ \\
% & COSOC \cite{cosoc_nips}& $73.57 \pm 0.43$ & $87.57 \pm 0.10$ \\
& ASCM \cite{ascm_pr}& $69.35 \pm 0.65$ & $84.87 \pm 0.58$ \\
& CPEA \cite{ceta_iccv}& $76.93 \pm 0.70$ & ${90.12 \pm 0.45}$ \\
& LA-PID \cite{lapid_tpami}& $68.10 \pm 0.50$ & ${90.60 \pm 0.50}$ \\
% & EPNet \cite{epnet_jmlr22}& $78.50 \pm 0.91$ & $88.36 \pm 0.57$ \\
& HELA-VFA \cite{helavfa_wacv}& $72.50 \pm 0.50$ & $87.60 \pm 0.10$ \\
& MetaDiff \cite{metadiff_aaai}& $64.99 \pm 0.77$ & $81.21 \pm 0.56$ \\
% & SemFew \cite{semfew_cvpr}& $77.63 \pm 0.63$ & $83.04 \pm 0.48$ \\
& FA-adapter \cite{faadapter_pr}& $68.36 \pm 0.45$ & $86.49 \pm 0.50$ \\
& SemFew-Trans \cite{semfew_cvpr}& $78.94 \pm 0.66$ & $86.49 \pm 0.50$ \\
\cdashline{2-4}
& Euclidean & ${77.71 \pm 0.89}$ & $88.49 \pm 0.57$ \\
& sDTW & \underline{$81.54 \pm 0.81$} & \underline{$92.37 \pm 0.47$} \\
& uDTW & $\mathbf{83.44 \pm 0.78}$ & $\mathbf{93.07 \pm 0.46}$ \\

\bottomrule
\end{tabular}}
\vspace{-0.5cm}
\end{table}

\begin{table}[tbp]
\centering
\caption{Few-shot classification accuracy (\%) on fine-grained datasets. % Results are reported for 1-shot and 5-shot settings. Best and second-best results are highlighted in bold and underline, respectively.
}
\vspace{-0.2cm}
\label{tab:finegrained}
\resizebox{\linewidth}{!}{\begin{tabular}{llcc}
\toprule
Dataset & Method & 1-shot & 5-shot \\
\midrule
\multirow{12}{*}{CUB-200-2011}
& AFS-FR \cite{afsfr_pr}& $\mathbf{94.40 \pm 0.13}$ & $\mathbf{97.96 \pm 0.05}$ \\
% & NegCosIC \cite{negcosic} & ${85.68 \pm 0.78}$ & $94.66 \pm 0.41$ \\ 
& CORE \cite{core_tpami}& $82.36 \pm 0.41$ & $91.89 \pm 0.30$ \\
& BSFA \cite{bfsa_tcsvt}& $82.27 \pm 0.46$ & $90.76 \pm 0.26$ \\
% & MCNet \cite{mcnet_nn}& $81.72 \pm 0.43$ & $92.62 \pm 0.23$ \\
% & RENet \cite{renet_iccv}& $79.49 \pm 0.44$ & $91.11 \pm 0.24$ \\
& MLCN \cite{mlcn_icme}& $77.96 \pm 0.44$ & $91.20 \pm 0.24$ \\
& Sun et al. \cite{egal_2026}& $75.14$ & $88.87$ \\
& S2M2 \cite{s2m2_wacv}& $72.92 \pm 0.83$ & $86.55 \pm 0.51$ \\
& RSaD \cite{rsad_tm}& $71.15$ & $84.03$ \\
& T2L \cite{t2l_kbs}& $71.04$ & $83.44$ \\
& TripletMAML \cite{triplemaml_tai}& $70.46 \pm 0.17$ & $81.43 \pm 0.86$ \\
\cdashline{2-4}
& Euclidean & $72.00 \pm 0.93$ & $85.61 \pm 0.59$ \\
& sDTW & $84.17 \pm 0.77$ & $93.74 \pm 0.44$ \\
& uDTW & \underline{$87.77 \pm 0.68$} & \underline{$95.08 \pm 0.38$} \\

\midrule
\multirow{6}{*}{Stanford Dogs}
& T2L \cite{t2l_kbs}& $52.12$ & $70.83$ \\
& FOT \cite{fot_nc}& $49.32 \pm 0.74$ & $68.18 \pm 0.69$ \\
& RSaD \cite{rsad_tm}& $73.75 \pm 0.93$ & $86.65 \pm 0.54$ \\
\cdashline{2-4}
& Euclidean & $71.48 \pm 0.95$ & $86.31 \pm 0.54$ \\
& sDTW & \underline{$75.32 \pm 0.80$} & \underline{$89.97 \pm 0.45$} \\
& uDTW & $\mathbf{83.38 \pm 0.73}$ & $\mathbf{93.58 \pm 0.37}$ \\

\midrule
\multirow{7}{*}{Stanford Cars}
% & FOT \cite{fot_nc}& $54.55 \pm 0.73$ & $73.69 \pm 0.65$ \\
& T2L \cite{t2l_kbs}& $56.80$ & $74.10$ \\
& BTG-Net \cite{btgnet_acmmm}& $\underline{90.28 \pm 0.34}$ & $\mathbf{96.78 \pm 0.15}$ \\
& BSFA \cite{bfsa_tcsvt}& $88.93 \pm 0.38$ & $95.20 \pm 0.20$ \\
& RSaD \cite{rsad_tm}& $87.27 \pm 0.70$ & $95.01 \pm 0.49$ \\
\cdashline{2-4}
& Euclidean & $68.34 \pm 0.92$ & $83.42 \pm 0.65$ \\
& sDTW & $82.58 \pm 0.83$ & $94.82 \pm 0.38$ \\
& uDTW & $\mathbf{90.45 \pm 0.63}$ & \underline{$96.64 \pm 0.31$} \\

\midrule
\multirow{7}{*}{Aircraft}
& BSNet \cite{bsnet_tip}& $64.83 \pm 1.00$ & $80.25 \pm 0.67$ \\
& FRN \cite{frn_cvpr21}& 70.17 & 83.81 \\
& CTX \cite{ctx_nips20}& 65.60 & 80.20 \\
& LE-ProtoPNet \cite{leprotopnet_tip26} & \underline{82.82} & \underline{89.13} \\
\cdashline{2-4}
& Euclidean & $62.58 \pm 1.00$ & $76.13 \pm 0.79$ \\
& sDTW & $74.02 \pm 1.02$ & $88.03 \pm 0.67$ \\
& uDTW & $\mathbf{83.89 \pm 0.88}$ & $\mathbf{91.68 \pm 0.56}$ \\

\midrule
\multirow{6}{*}{NABirds}
& CovaMNet \cite{covamnet_aaai19} & $66.29 \pm 0.82$ & $82.54 \pm 0.87$ \\
& LRPABN \cite{lrpabn_acmmm21}& $67.73 \pm 0.81$ & $81.62 \pm 0.58$\\
& PACNet \cite{pacnet_tmm24}& $75.30 \pm 0.90$ & $88.20 \pm 0.60$ \\
\cdashline{2-4}
& Euclidean & $71.75 \pm 1.32$ & $76.17 \pm 1.52$ \\
& sDTW & \underline{$81.19 \pm 0.79$} & \underline{$92.85 \pm 0.42$} \\
& uDTW & $\mathbf{85.72 \pm 0.72}$ & $\mathbf{94.49 \pm 0.37}$ \\
\bottomrule
\end{tabular}}
\vspace{-0.2cm}
\end{table}

\begin{table}[tbp]
\caption{Few-shot accuracy (\%) on UFGVC benchmarks. % Best results are in bold.
}
\vspace{-0.2cm}
\centering
\label{tab:untra-finegrained}
\resizebox{\linewidth}{!}{\begin{tabular}{lllcc}
\toprule
& Dataset & Method & 1-shot & 2-shot \\
\midrule
\multirow{6}{*}{Supervised} & \multirow{3}{*}{Cotton80}
& Euclidean & $60.65 \pm 1.40$ & $71.07 \pm 1.60$ \\
& & sDTW & $64.02 \pm 1.38$ & $71.40 \pm 1.61$ \\
& & uDTW & $\mathbf{64.72 \pm 1.35}$ & $\mathbf{75.47 \pm 1.54}$ \\
\cline{2-5}
& \multirow{3}{*}{SoyLocal}
& Euclidean & $59.40 \pm 1.36$ & $67.63 \pm 1.67$ \\
& & sDTW & $58.85 \pm 1.34$ & $71.43 \pm 1.62$ \\
& & uDTW & $\mathbf{61.25 \pm 1.35}$ & $\mathbf{72.67 \pm 1.57}$ \\
\midrule
\multirow{6}{*}{Unsupervised} & \multirow{3}{*}{Cotton80}
& Euclidean & $51.77 \pm 1.29$ & $60.73 \pm 1.65$ \\
& & sDTW & $58.48 \pm 1.25$ & $69.27 \pm 1.59$ \\
& & uDTW & $\mathbf{60.12 \pm 1.27}$ & $\mathbf{72.30 \pm 1.57}$ \\
\cline{2-5}
& \multirow{3}{*}{SoyLocal}
& Euclidean & $50.53 \pm 1.45$ & $58.60 \pm 1.77$ \\
& & sDTW & $54.13 \pm 1.42$ & $64.30 \pm 1.67$ \\
& & uDTW & $\mathbf{56.78 \pm 1.40}$ & $\mathbf{64.83 \pm 1.69}$ \\
\bottomrule
\end{tabular}}
\vspace{-0.2cm}
\end{table}

\begin{table}[tbp]
\caption{Unsupervised few-shot classification accuracy (\%) on Omniglot and miniImageNet. % ($N$-way $K$-shot). % Best and second-best are in bold and underline.
}
\vspace{-0.2cm}
\centering
\resizebox{\linewidth}{!}{\begin{tabular}{lcccccccc}
\toprule
\multirow{3}{*}{Method} & \multicolumn{4}{c}{{Omniglot}} & \multicolumn{4}{c}{{miniImageNet}} \\
\cmidrule(lr){2-5} \cmidrule(lr){6-9}
& (5,1) & (5,5) & (20,1) & (20,5) & (5,1) & (5,5) & (5,20) & (5,50) \\
\midrule

CACTUs-MAML \cite{cactu_iclr19}
& 68.84 & 87.78 & 48.09 & 73.36
& 39.90 & 53.97 & 63.84 & 69.64 \\

CACTUs-ProtoNets \cite{cactu_iclr19}
& 68.12 & 83.58 & 47.75 & 66.27
& 39.18 & 53.36 & 61.54 & 63.55 \\

UMTRA \cite{umtra_nips19}
& \underline{83.80} & \textbf{95.43} & \textbf{74.25} & \textbf{92.12}
& 39.93 & 50.73 & 61.11 & 67.15 \\

LASIUM-MAML \cite{lasium_iclr21}
& 83.26 & \underline{95.29} & -- & --
& 40.19 & 54.56 & 65.17 & 69.13 \\

LASIUM-ProtoNets \cite{lasium_iclr21}
& 80.15 & 91.10 & -- & --
& 40.05 & 52.53 & 61.09 & 64.89 \\
\hdashline
Euclidean
& 82.20 & 93.09 & 62.61 & 81.32
& 76.92  & 91.21 & 93.93 & 94.47 \\

sDTW
& 83.13 & 94.55 & 67.69 & 87.30
& \underline{79.29} & \underline{92.96} & \underline{96.36} & \underline{96.96} \\

uDTW
& \textbf{83.84} & {94.88} & \underline{68.21}  & \underline{87.48}
& \textbf{80.18} & \textbf{93.16} & \textbf{96.55} & \textbf{97.13} \\

\bottomrule
\end{tabular}}
\label{tab:unsup-generic}
\vspace{-0.5cm}
\end{table}

\begin{table}[tbp]
\vspace{-0.5cm}
\centering
\caption{Cross-domain few-shot classification accuracy (\%) on CUB-200-2011, Stanford Cars, Cotton80, and SoyLocal. % ($N$-way $K$-shot).
}
\vspace{-0.2cm}
\resizebox{\linewidth}{!}{\begin{tabular}{lcccccccc}
\toprule
\multirow{2}{*}{Method} & \multicolumn{2}{c}{{CUB-200-2011}} & \multicolumn{2}{c}{{Stanford Cars}}& \multicolumn{2}{c}{{Cotton80}} & \multicolumn{2}{c}{{SoyLocal}} \\
\cmidrule(lr){2-3} \cmidrule(lr){4-5} \cmidrule(lr){6-7} \cmidrule(lr){8-9}
 & (5,5) & (5,20) & (5,5) & (5,20) & (5,1) & (5,2) & (5,1) & (5,2)\\
\midrule
Meta-GMVAE  \cite{metagmvae_iclr2021}& 47.48 & 54.08 & 31.39 & 38.36 & -- & -- & -- & --\\
Meta-SVEBM  \cite{svebm_nipsw21}& 45.50 & 54.61 & 34.27 & 46.23 & --& --& --& --\\
PsCo        \cite{psco_iclr23}& 57.38 & 68.58 & 44.01 & 57.50 & --& --& --& --\\
\hdashline
Euclidean    & 69.20 & 76.94 & 57.55 & 68.70 & 58.18 & 67.43 & 55.08 & 62.30 \\
sDTW        & 78.73 & 87.25 & 65.31 & 79.79 & 58.17 & 67.57 & 55.05 & 62.10 \\
uDTW        & \textbf{79.20} & \textbf{89.12} & \textbf{69.56} & \textbf{79.96} & \textbf{58.27} & \textbf{68.47} & \textbf{55.70} & \textbf{62.37}\\
\bottomrule
\end{tabular}}
\label{tab:unsup-cross-domain}
\vspace{-0.3cm}
\end{table}

\subsection{Few-shot Action Recognition}

% We evaluate uDTW in few-shot action recognition, where actions are classified based on similarity to a small number of labeled support examples.

We evaluate uDTW for few-shot action recognition, classifying actions from a few labeled support examples.

%\vspace{0.1cm}
\noindent
\textbf{Setup.} For NTU-120, we follow the standard one-shot protocol~\cite{Liu_2019_NTURGBD120}. Based on this, we construct a similar one-shot protocol for NTU-60, using 50 action classes for training and 10 for testing. We also evaluate on both 2D and 3D Kinetics-skeleton, splitting 200 actions for training and using the remainder for testing. Each video block is reshaped and resized to a $224 \times 224$ color image, and fed into MatchNets and ProtoNet to learn feature representations. % We compare uDTW with Euclidean, sDTW, sDTW div., and TAP.

%\vspace{0.1cm}
\noindent
\textbf{Quantitative results.} Tables~\ref{tab:ntu_combined} and~\ref{tab:kinetics} show that uDTW consistently outperforms sDTW and sDTW div. in both supervised and unsupervised settings. On Kinetics-skeleton, uDTW improves performance on 3D skeletons by 2.4\% (supervised) and 4.4\% (unsupervised). In the supervised setting, uDTW outperforms TAP by approximately 4\% on NTU-60 and 2\% on NTU-120. In the unsupervised setting, it surpasses sDTW by about 2\% and 3\% on NTU-60 and NTU-120, respectively.

\subsection{Few-shot Image Classification}

% \textbf{Cotton80 \& SoyLocal (Evaluation Protocol).} We adopt the Cotton80 and SoyLocal benchmarks from UFGVC dataset~\cite{ufgvc} for few-shot evaluation. These UFGVC datasets are highly data-scarce, containing only 240 samples across 80 classes (i.e., 3 samples per class). To accommodate this limitation, we follow a constrained episodic evaluation protocol, using 5-way 1-shot with 2 query samples per class, and 5-way 2-shot with 1 query sample per class.

% For the base/novel class partition, we use a deterministic splitting strategy. Specifically, we first fix the random seed to 0, and then sort all classes in a predefined order. Classes at odd indices are assigned to the novel set, while those at even indices are assigned to the base set. Once either set reaches the target size, all remaining classes are assigned to the other set. This results in a balanced split of 40 novel and 40 base classes, and no validation set is used.

% We evaluate uDTW on a range of few-shot supervised benchmarks, including generic image classification (CIFAR-FS, miniImageNet, tieredImageNet), fine-grained datasets (CUB-200-2011, Stanford Dogs, Stanford Cars, Aircraft, NABirds), and ultra-fine-grained visual classification (UFGVC) datasets (Cotton80, SoyLocal). 

\subsubsection{Quantitative Results}

We evaluate uDTW across generic, fine-, and ultra-fine-grained few-shot image classification.

%\vspace{0.1cm}
\noindent
\textbf{Supervised few-shot classification.} On generic benchmarks (Table \ref{tab:generic}), uDTW consistently achieves the highest accuracy for both 1-shot and 5-shot tasks, improving over the second-best methods (typically sDTW) by 0.5–2\%. This demonstrates its strength in capturing intra-class structural variations and aligning features across limited support samples. On fine-grained datasets (Table \ref{tab:finegrained}), uDTW generally outperforms prior methods, particularly in 1-shot scenarios. However, it occasionally ranks second, for example on Stanford Cars (compared to BTG-Net) and CUB-200-2011 (compared to AFS-FR), likely because these methods use highly discriminative pre-trained embeddings that already separate classes effectively, leaving limited room for additional gains from alignment. On UFGVC benchmarks (Table \ref{tab:untra-finegrained}), uDTW consistently attains the best results, though margins over sDTW are modest, reflecting the challenge of extremely low-data, ultra-fine-grained categories where subtle inter-class differences constrain further improvements. These patterns indicate that uDTW provides the largest benefit when intra-class variability is high, the number of support examples is small, or embeddings alone are insufficient for class separation.

\begin{figure*}[tbp]
\vspace{-0.2cm}
  \centering
  \includegraphics[width=\linewidth]{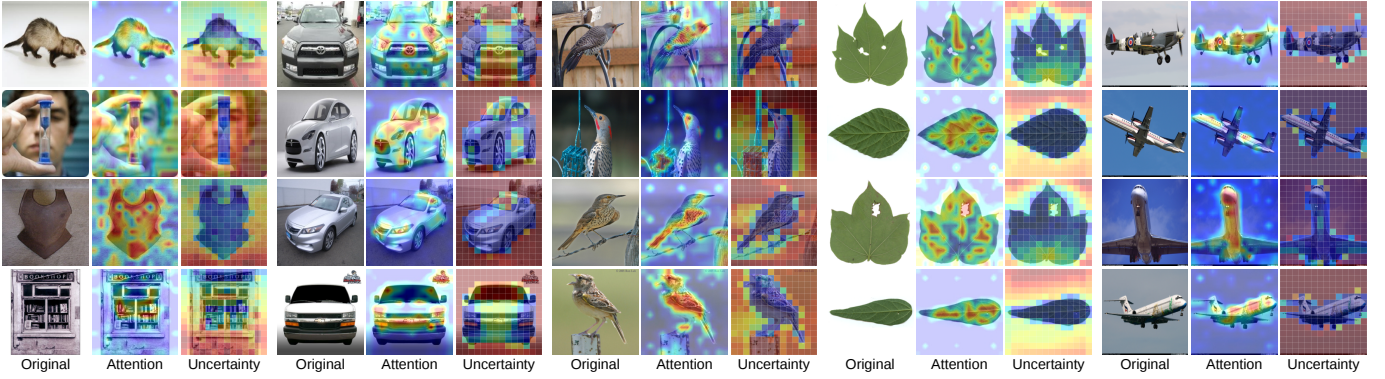}
  \vspace{-0.6cm}
  \caption{Comparison between DAAM attention and the learned uncertainty across MiniImageNet, Stanford Cars, CUB-200-2011, UFGVC, and Aircraft. 
  Each triplet shows the input image, DAAM attention (brighter = higher importance), and uDTW uncertainty (blue = low, red = high). Low-uncertainty regions consistently coincide with high-attention responses, evidencing the reverse-attention effect induced by our uncertainty-aware alignment, where reliable (semantically informative) regions are emphasized while ambiguous ones are suppressed.}
  \label{fig:vis-attn-uncert}
  \vspace{-0.3cm}
\end{figure*}

%\vspace{0.1cm}
\noindent
\textbf{Unsupervised few-shot classification.} 
Across generic benchmarks (Table \ref{tab:unsup-generic}), uDTW consistently attains the best or second-best accuracy across all % $N$-way $K$-shot 
settings, outperforming prior methods such as UMTRA and LASIUM. The improvements are particularly notable for higher-way and higher-shot tasks on miniImageNet, reflecting uDTW’s ability to align feature sequences effectively even without labels. On Omniglot, UMTRA slightly surpasses uDTW in some high-shot settings, likely because it generates synthetic tasks tailored to the dataset’s character-level structure, reducing the incremental benefit of alignment. In cross-domain scenarios (Table \ref{tab:unsup-cross-domain}), uDTW achieves the highest accuracy in all evaluated tasks, with the largest gains observed on Stanford Cars, where domain shift is more severe. This suggests that uDTW’s sequence alignment is especially effective for mitigating distribution mismatch, though on CUB-200-2011 the margin over sDTW is smaller, possibly because the domain shift is less extreme and embeddings already provide a strong prior.

\subsubsection{Qualitative Results} Fig. \ref{fig:vis-attn-uncert} highlights a consistent correspondence between token-level uncertainty and attention across datasets. Regions with low uncertainty align closely with high-response areas in DAAM attention maps, while less informative or ambiguous regions exhibit higher uncertainty. This shows a clear reverse-attention effect induced by our uncertainty-aware alignment, where semantically meaningful regions are automatically emphasized through low uncertainty, without relying on an explicit attention mechanism. Compared to attention, the learned uncertainty provides a more direct measure of correspondence reliability, leading to cleaner and more interpretable token importance. These observations support our formulation, demonstrating that uDTW captures semantic relevance via uncertainty and offers a principled alternative to similarity-based attention.

\section{Conclusion}

We introduced uncertainty-aware alignment, a probabilistic generalization of DTW that models heteroscedastic uncertainty in pairwise correspondences. Derived from a maximum-likelihood formulation, the resulting objective combines precision-weighted matching with log-variance regularization, enabling robust and interpretable alignment under noise and ambiguity.
The soft relaxation induces a Gibbs like distribution over alignment paths, forming structured coupling. This coupling can be viewed as a path-constrained transport plan and a globally consistent attention mechanism, where uncertainty acts as a reliability-aware modulation that suppresses ambiguous matches.
We further extended the framework to tokenized visual representations, where learned uncertainty induces a reverse-attention effect: informative regions exhibit low uncertainty and dominate the alignment, while noisy regions are attenuated. Experiments across diverse domains demonstrate consistent improvements and reveal that uncertainty correlates with semantic importance.

\bibliographystyle{IEEEtran}
\bibliography{eccv22}

% \newpage

\vskip -2.0\baselineskip plus -1fil

\begin{IEEEbiography}
[{\includegraphics[width=1in,height=1.25in,clip,keepaspectratio]{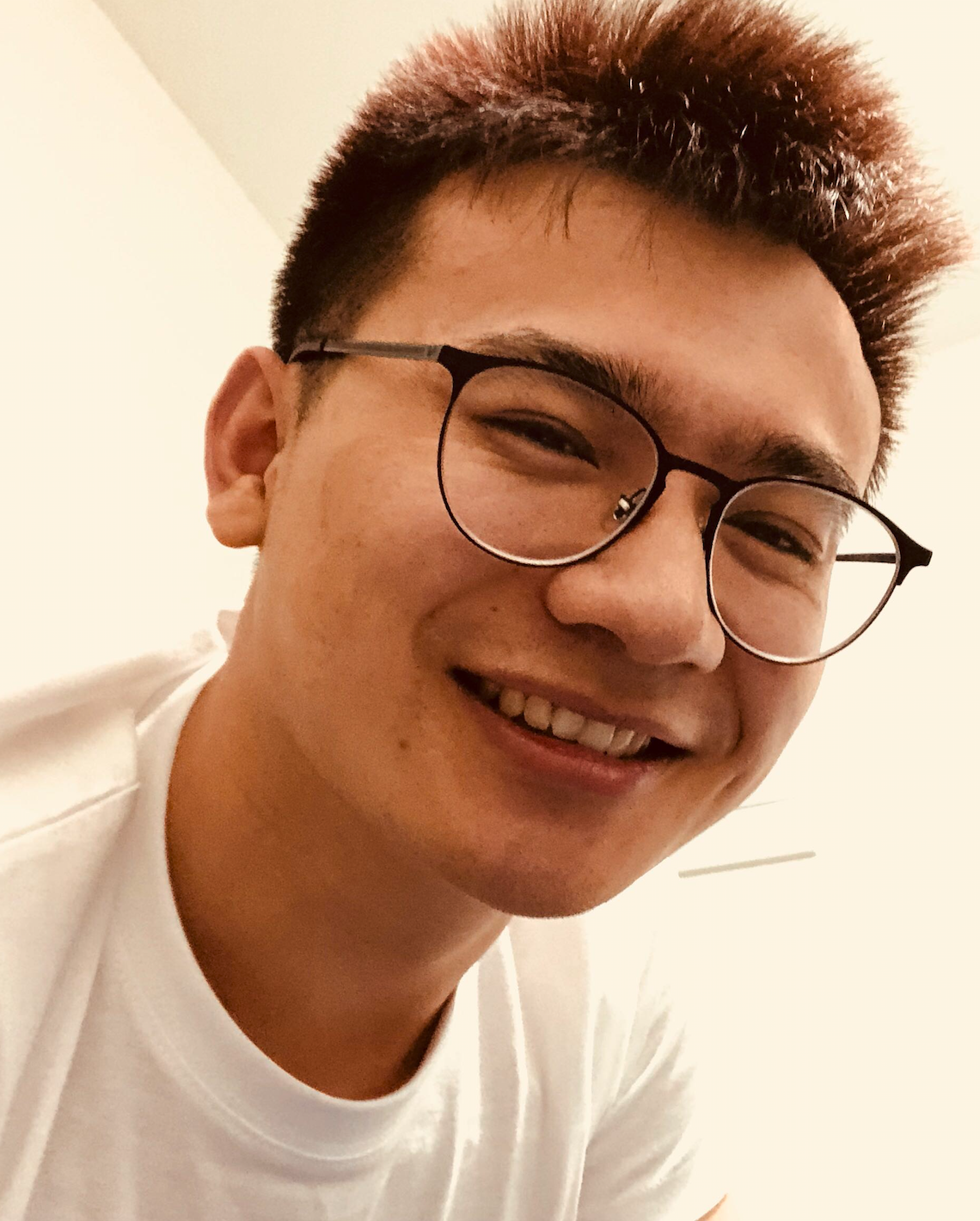}}]{Lei Wang} received his M.E. in Software Engineering from the University of Western Australia (UWA) in 2018 and his Ph.D. in Engineering and Computer Science from the Australian National University (ANU) in 2023. He is a Research Fellow in the School of Electrical and Electronic Engineering at Griffith University and a Visiting Scientist with Data61/CSIRO. He leads the Temporal Intelligence and Motion Extraction (TIME) Lab at Griffith University. He previously held research positions at ANU, UWA, and Data61/CSIRO. His research focuses on motion-, data-, and model-centric approaches to action recognition and anomaly detection. 
He has authored numerous first-author papers in top-tier venues, including CVPR, ICCV, ECCV, ACM Multimedia, NeurIPS, ICLR, ICML, AAAI, TPAMI, IJCV, and TIP, and received the Sang Uk Lee Best Student Paper Award at ACCV 2022. He serves as an Area Chair for NeurIPS 2026, ACM Multimedia 2024-2026, ICASSP 2025, and ICPR 2024, and was recognized as an Outstanding Area Chair at ACM Multimedia 2024.
\end{IEEEbiography}

\vskip -2.0\baselineskip plus -1fil

\begin{IEEEbiography}
[{\includegraphics[trim=0 0 0 300, width=1in,height=1.25in,clip,keepaspectratio]{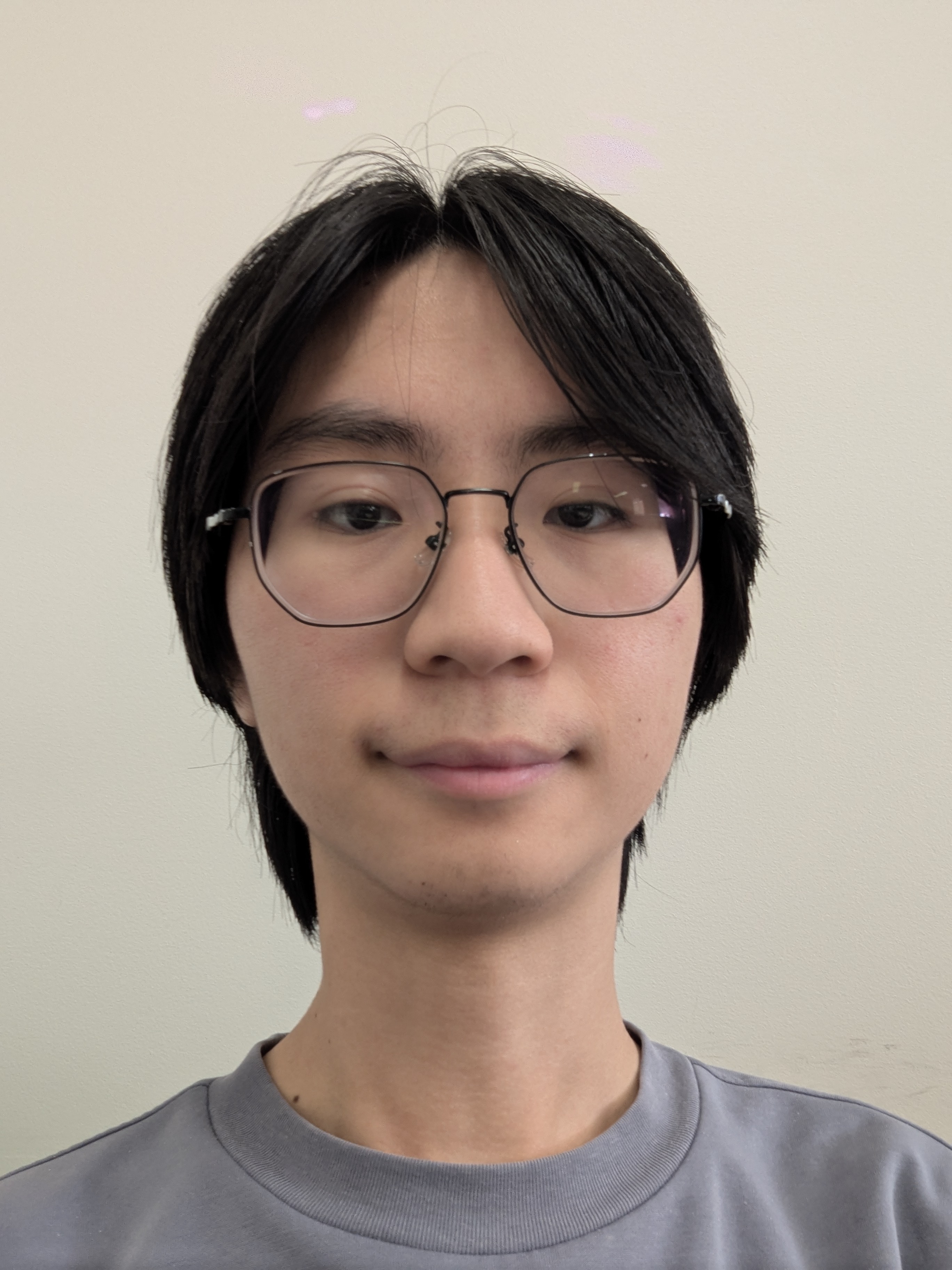}} % {\includegraphics[trim=0 150 0 160, width=1in,height=1.25in,clip,keepaspectratio]{bibs/syuan-hao.JPG}}
]{Syuan-Hao Li} received the B.S. degree in Computer Science, National Taitung University (NTTU), Taiwan, in 2025. He is currently a Ph.D. pathway student at Griffith University and a research intern at the Temporal Intelligence and Motion Extraction (TIME) Lab. He serves as a workshop coordinator for TIME 2026: the 2nd International Workshop on Transformative Insights in Multi-faceted Evaluation, hosted at the Web Conference (WWW 2026). His research interests include temporal modeling, multimodal intelligence, and fine- and ultra-fine-grained visual understanding.
\end{IEEEbiography}

\vskip -2.0\baselineskip plus -1fil

\begin{IEEEbiography}[{\includegraphics[trim=0 0 0 0, width=1in,height=1.25in,clip,keepaspectratio]{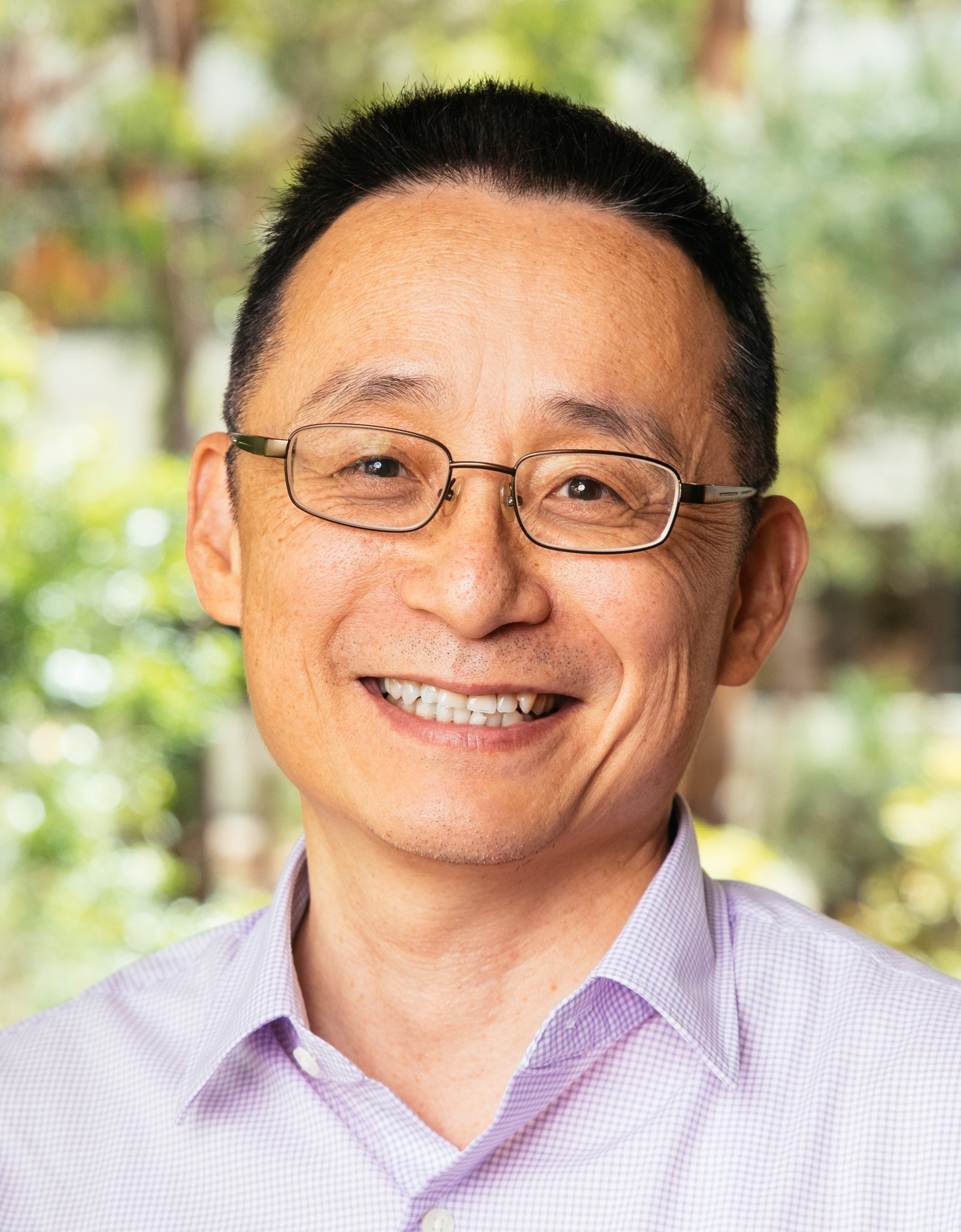}}% {\includegraphics[trim=40 50 30 5, width=1in,height=1.25in,clip,keepaspectratio]{bibs/yongsheng.png}}
]{Yongsheng Gao} received the BSc and MSc degrees in Electronic Engineering from Zhejiang University, China, in 1985 and 1988, respectively, and the PhD degree in Computer Engineering from Nanyang Technological University, Singapore. He is currently a Professor with the School of Engineering and Built Environment, Griffith University, and Director of the ARC Research Hub for Driving Farming Productivity and Disease Prevention, Australia. He was previously the Leader of the Biosecurity Group at the Queensland Research Laboratory, National ICT Australia (ARC Centre of Excellence), a consultant at Panasonic Singapore Laboratories, and an Assistant Professor at Nanyang Technological University. His research interests include smart farming, machine vision for agriculture, biosecurity, face recognition, biometrics, image retrieval, computer vision, pattern recognition, environmental informatics, and medical imaging. He is a recipient of the 2025 ARC Industry Laureate Fellow.\end{IEEEbiography}

\vskip -2.0\baselineskip plus -1fil

\begin{IEEEbiography}[{\includegraphics[trim=80 80 30 0,width=1in,height=1.25in,clip,keepaspectratio]{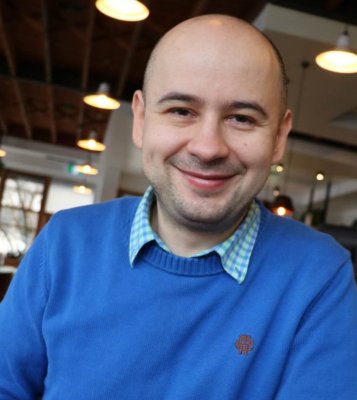}}]{Piotr Koniusz} received the BSc degree in Telecommunications and Software Engineering from Warsaw University of Technology, Poland, in 2004, and the PhD degree in Computer Vision from CVSSP, University of Surrey, U.K., in 2013. He is an Associate Professor in Theoretical ML at the University of New South Wales (UNSW) and a Principal (now Visiting) Researcher with the Machine Learning Research Group, Data61/CSIRO. He was previously a postdoctoral researcher with the LEAR team at INRIA, France. 
His research interests include representation learning (contrastive and self-supervised learning, unlearning), vision-language models, MLLMs, and deep and graph neural networks, as well as Machine Learning Safety. He has received awards including the Sang Uk Lee Best Student Paper Award (ACCV 2022), Runner-up APRS/IAPR Best Student Paper Award (DICTA 2022), and Outstanding Area Chair recognition (ICLR 2021--2023). He served as a Program Chair for NeurIPS 2025 and serves as a Senior Workshop Program Chair for NeurIPS 2026, and a Journal Track Chair for ACML 2026.\end{IEEEbiography}

\end{document}